\documentclass{article} % For LaTeX2e
\usepackage{amsmath,amsfonts,bm}

\def\eqref#1{equation~\ref{#1}}
\def\1{\bm{1}}

\DeclareMathAlphabet{\mathsfit}{\encodingdefault}{\sfdefault}{m}{sl}
\SetMathAlphabet{\mathsfit}{bold}{\encodingdefault}{\sfdefault}{bx}{n}

\newcommand{\R}{\mathbb{R}}

\DeclareMathOperator*{\argmax}{arg\,max}
\DeclareMathOperator*{\argmin}{arg\,min}

\usepackage{microtype}
\usepackage{graphicx}
\usepackage{booktabs} % for professional tables

\usepackage{hyperref}

\usepackage[accepted]{icml2025}

\usepackage{amsthm}
\usepackage{algorithm,algorithmic}
\renewcommand{\algorithmiccomment}[1]{\bgroup\hfill//~#1\egroup}
\usepackage{amsmath}
\usepackage{amssymb}
\usepackage{mathtools}
\usepackage{enumitem}
\setlist[itemize]{noitemsep, topsep=0pt, leftmargin=11pt}
\setlist[enumerate]{noitemsep, topsep=0pt, leftmargin=11pt}
\usepackage{wrapfig}
\usepackage{xcolor}
\usepackage{subcaption}
\usepackage{wrapfig}
\usepackage[capitalize,noabbrev]{cleveref}
\usepackage{soul}

\usepackage[capitalize,noabbrev]{cleveref}

\theoremstyle{plain}
\newtheorem{theorem}{Theorem}[section]

\newtheorem{lemma}[theorem]{Lemma}

\theoremstyle{definition}
\newtheorem{definition}[theorem]{Definition}

\theoremstyle{remark}

\def\1{\mathbf{1}}
\def\P{\mathbb{P}}
\def\R{\mathbb{R}}

\newcommand{\mc}[1]{\mathcal{#1}}

\usepackage[textsize=tiny]{todonotes}

\icmltitlerunning{Improved Algorithm for Deep Active Learning under Imbalance via Optimal Separation}

\begin{document}
\twocolumn[
\icmltitle{Improved Algorithm for Deep Active Learning under Imbalance via Optimal Separation}

% It is OKAY to include author information, even for blind
% submissions: the style file will automatically remove it for you
% unless you've provided the [accepted] option to the icml2025
% package.

% List of affiliations: The first argument should be a (short)
% identifier you will use later to specify author affiliations
% Academic affiliations should list Department, University, City, Region, Country
% Industry affiliations should list Company, City, Region, Country

% You can specify symbols, otherwise they are numbered in order.
% Ideally, you should not use this facility. Affiliations will be numbered
% in order of appearance and this is the preferred way.
\icmlsetsymbol{equal}{*}

\begin{icmlauthorlist}
\icmlauthor{Shyam Nuggehalli}{equal,wisc}
\icmlauthor{Jifan Zhang}{equal,wisc}
\icmlauthor{Lalit Jain}{uw}
\icmlauthor{Robert Nowak}{wisc}
\end{icmlauthorlist}

\icmlaffiliation{wisc}{University of Wisconsin-Madison}
\icmlaffiliation{uw}{University of Washington, Seattle}

\icmlcorrespondingauthor{Jifan Zhang}{jifan@cs.wisc.edu}

% You may provide any keywords that you
% find helpful for describing your paper; these are used to populate
% the "keywords" metadata in the PDF but will not be shown in the document
\icmlkeywords{Deep Learning, Active Labeling, Class Imbalance}

\vskip 0.3in
]

% this must go after the closing bracket ] following \twocolumn[ ...

% This command actually creates the footnote in the first column
% listing the affiliations and the copyright notice.
% The command takes one argument, which is text to display at the start of the footnote.
% The \icmlEqualContribution command is standard text for equal contribution.
% Remove it (just {}) if you do not need this facility.

%\printAffiliationsAndNotice{}  % leave blank if no need to mention equal contribution
\printAffiliationsAndNotice{\icmlEqualContribution} % otherwise use the standard text.

\begin{abstract}%   <- trailing '%' for backward compatibility of .sty file
Class imbalance severely impacts machine learning performance on minority classes in real-world applications. While various solutions exist, active learning offers a fundamental fix by strategically collecting balanced, informative labeled examples from abundant unlabeled data. We introduce DIRECT, an algorithm that identifies class separation boundaries and selects the most uncertain nearby examples for annotation. By reducing the problem to one-dimensional active learning, DIRECT leverages established theory to handle batch labeling and label noise -- another common challenge in data annotation that particularly affects active learning methods. Our work presents the first comprehensive study of active learning under both class imbalance and label noise. Extensive experiments on imbalanced datasets show DIRECT reduces annotation costs by over 60\% compared to state-of-the-art active learning methods and over 80\% versus random sampling, while maintaining robustness to label noise.
\end{abstract}

\section{Introduction}
Large-scale deep learning models are playing increasingly important roles across many industries. Human feedback and annotations have played a significant role in developing such systems. Progressively over time, we believe the role of humans in a machine learning pipeline will shift to annotating rare yet important cases. However, under data imbalance, the typical strategy of randomly choosing examples for annotation becomes especially inefficient. This is because the majority of the labeling budget would be spent on common and well-learned classes, resulting in insufficient rare class examples for training an effective model. To mitigate this issue, many recent active learning algorithms have focused on labeling more class-balanced and informative examples \citep{aggarwal2020active,kothawade2021similar,zhang2022galaxy,zhang2024algorithm,soltani2024learning}. For many large-scale annotation jobs, this challenge of data imbalance is further compounded by label noise -- a critical and common issue that results from annotator decision fatigue and perception differences. A rich body of literature on agnostic active learning \citep{balcan2006agnostic,dasgupta2007general,hanneke2014theory,katz2021improved} addresses this challenge on low-complexity model classes (e.g. linear models). However, for deep learning models, these algorithms often becomes ineffective due to the large model class complexity. In this paper, we propose a novel active learning strategy for both class imbalance and label noise. Our algorithm DIRECT sequentially and adaptively chooses informative and more class-balanced examples for annotation while being robust to noisy annotations. To the best of our knowledge, this is the first deep active learning study to address the challenging yet prevalent scenario where both imbalance and label noise coexist.
\begin{figure*}[t]
    \centering
    \begin{subfigure}[t]{.32\textwidth}
        \centering
        \includegraphics[width=\textwidth]{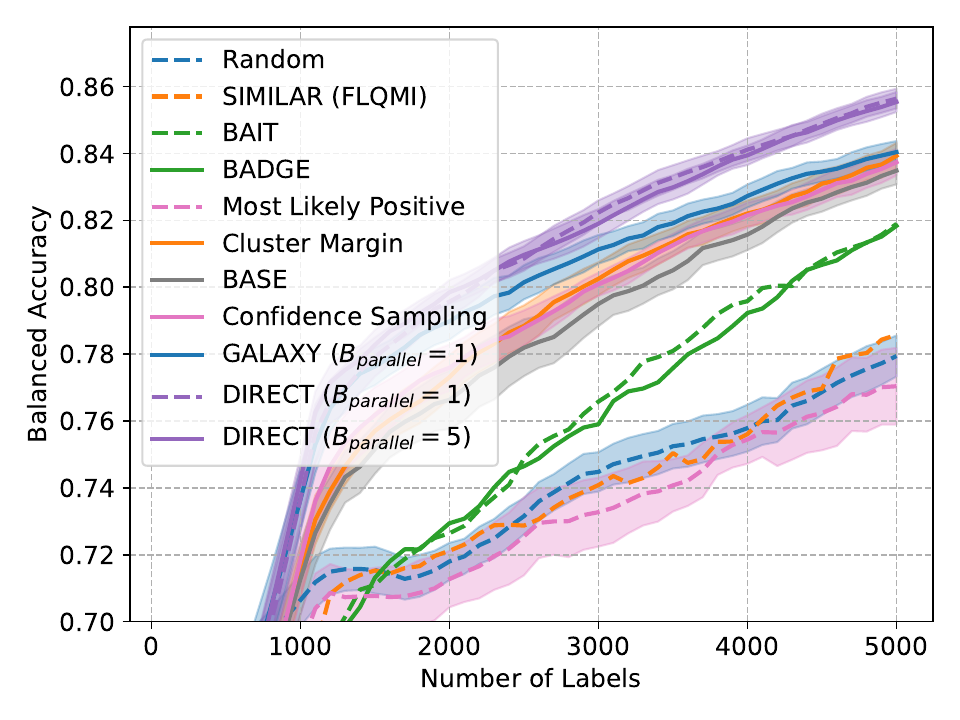}
        \caption{Imbalanced CIFAR-10, two classes, no label noise.}
    \end{subfigure}
    \begin{subfigure}[t]{.32\textwidth}
        \centering
        \includegraphics[width=\textwidth]{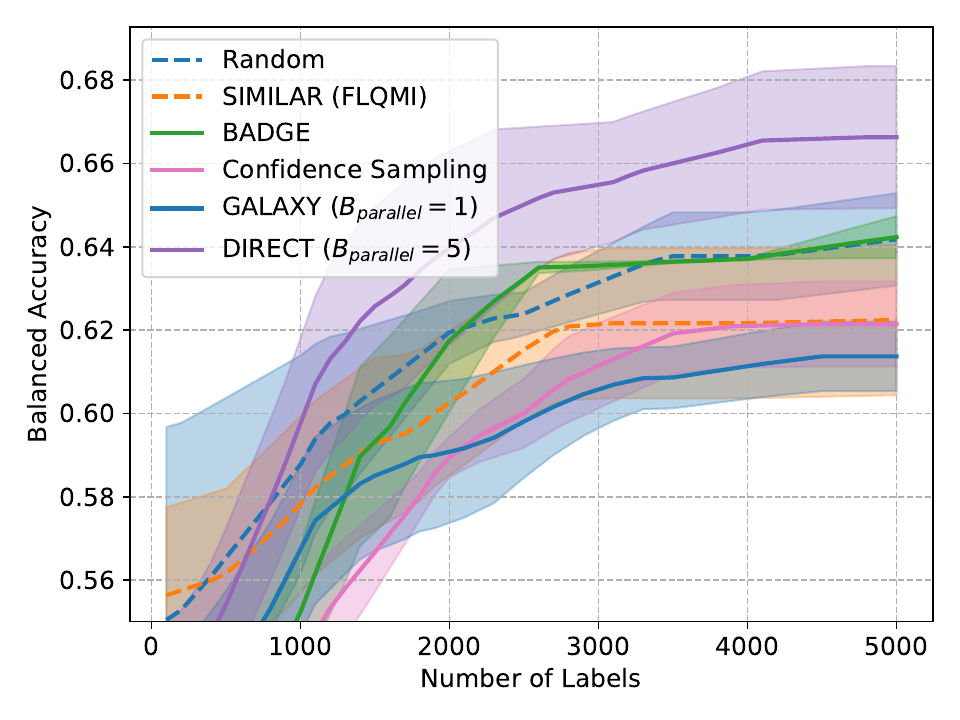}
        \caption{Imbalanced CIFAR-100, two classes, 20\% label noise.}
    \end{subfigure}
    \begin{subfigure}[t]{.32\textwidth}
        \centering
        \includegraphics[width=\textwidth]{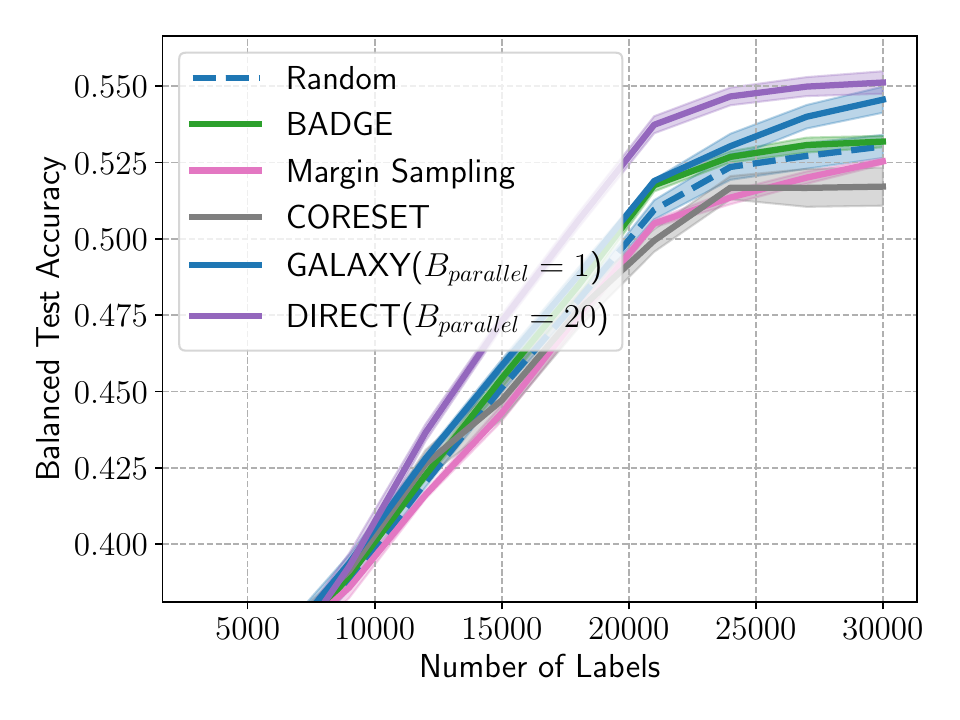}
        \caption{LabelBench FMoW (62 classes with imbalance), no label noise}
    \end{subfigure}
    \vspace{-0.5\intextsep}
    \caption{Performance of DIRECT over existing baselines for both noiseless and noisy settings. The x-axis represents the total number of labeled examples so far and the y-axis shows the neural network's balanced accuracy. Both (a) and (b) are using supervised training of ResNet-18. In (c), we finetune CLIP ViT-B32 model in combination of semi-supervised training under the LabelBench framework~\citep{zhang2024labelbench}. $B_\text{parallel}$ is the batch size indicating the number of parallel annotators. $B_\text{parallel} = 1$ indicates the synchronous annotation requirement by GALAXY. Our algorithm DIRECT takes pre-specified $B_\text{parallel}$ as input, which is determined by real world scenarios. }
    \label{fig:intro}
    \vspace{-1\intextsep}
\end{figure*}

To bridge the gap between the imbalanced deep active learning and the agnostic active learning literature, we propose a novel reduction of the imbalanced classification problem into a set of one-dimensional agnostic active learning problems. For each class, our reduction sorts unlabeled examples into an list ordered by one-vs-rest margin scores. The objective of DIRECT is to find the \emph{optimal separation threshold} which best separates the examples in the given class from the rest. By relating our problem to that of finding the best threshold classifier, we are able to employ ideas from the agnostic active learning literature to learn the separation threshold robustly under label noise. By annotating around the threshold, the annotated examples are more class-balanced and informative.

Comparing to existing active learning algorithms such as BADGE~\citep{ash2019deep}, Cluster-Margin~\citep{citovsky2021batch}, SIMILAR~\citep{kothawade2021similar}, GALAXY~\citep{zhang2022galaxy} and many others, DIRECT improves significantly in label efficiencies -- less annotations needed to reach the same accuracy. Notably, most existing methods mentioned above are proposed to handle batch labeling, while previous work by \citet{zhang2022galaxy} proposes a superior performance algorithm at the cost of only allowing one annotation at a time. Our algorithm DIRECT is able to obtain the best of both worlds -- practical scalability to large annotation jobs by batch labeling while also getting superior performance than all algorithms including GALAXY.
On imbalanced datasets, DIRECT achieves state-of-art label efficiency on both supervised fine-tuning of ResNet-18 and semi-supervised fine-tuning of large pretrained model under the LabelBench~\citep{zhang2024labelbench} framework.

To summarize our main contributions:
\begin{itemize}
    \item We propose a novel reduction that bridges the advancement in the theoretical agnostic active learning literature to imbalanced active classification for deep neural networks.
    \item Our algorithm DIRECT addresses the prevalent imbalance and label noise issues and annotates a more class-balanced and informative set of examples.
    \item Compared to state-of-art algorithm GALAXY \citep{zhang2022galaxy}, DIRECT allows parallel annotation by multiple annotators while still maintaining significant label-efficiency improvement.
    \item We conduct experiments across 12 dataset settings, four levels of label noise and for both ResNet-18 and large pretrained model (CLIP ViT-B32). DIRECT consistently outperforms existing baseline algorithms by saving more than $60\%$ annotation cost compared to the best existing algorithm, and more than $80\%$ annotation cost compared to random sampling.
\end{itemize}

\section{Related Work}
\textbf{Class-Balanced Deep Active Learning}
Active learning strategies sequentially and adaptively choose examples for annotation. Many uncertainty-based deep active learning methods extend the traditional active learning literature such as margin, least confidence and entropy sampling \citep{tong2001support, settles2009active, balcan2006agnostic, kremer2014active}. These methods have been shown to perform among the top when fine-tuning large pretrained models and combined with semi-supervised learning algorithms \citep{zhang2024labelbench}. More sophisticated methods have been proposed to optimize chosen examples' uncertainty \citep{gal2017deep,ducoffe2018adversarial,beluch2018power}, diversity \citep{sener2017active, geifman2017deep, citovsky2021batch}, or a mix of both \citep{ash2019deep, ash2021gone, wang2021deep, elenter2022lagrangian, mohamadi2022making}. However, these methods often perform poorly under prevalent and realistic scenarios such as label noises \citep{khosla2022neural} or class imbalance \cite{kothawade2021similar,zhang2022galaxy,zhang2024labelbench}.

\textbf{Deep Active Learning under Imbalance}
Data imbalance and rare instances are prevalent in almost all modern machine learning applications. Active learning techniques are effective in addressing the problem in its root by collecting a more class-balanced dataset \citep{aggarwal2020active, kothawade2021similar, emam2021active, zhang2022galaxy, coleman2022similarity, jin2022deep, cai2022active,zhang2024algorithm,xie2024deep}. To this end, \citet{kothawade2021similar} propose a submodular-based method that actively annotates examples similar to known examples of rare instances. GALAXY\citep{zhang2022galaxy} constructs one-dimensional linear graphs and applies graph-based active learning techniques in annotating a set of examples that are both class-balanced and uncertain. While GALAXY outperforms existing algorithms, due to a bisection procedure involved, it does not allow parallel annotation. In addition, bisection procedures are generally not robust against label noises, a prevalent challenge in real world annotation tasks. Our algorithm DIRECT mitigates all of the above shortcomings of GALAXY while outperforming it even with synchronous labeling and no label noise, beating GALAXY in its own game. Lastly, we distinguish our work from \citet{zhang2024algorithm}, where the paper studies the algorithm selection problem. Unlike our goal of proposing a new deep active learning algorithm, the paper proposes meta algorithms to choose the right active learning algorithm among a large number of candidate algorithms.

\textbf{Agnostic Active Learning for Label Noise}
Label noise for active learning has been primarily studied under the extensive literature on agnostic learning. We refer the interested reader to the survey~\citep{hanneke2014theory} for a thorough discussion. All of these works, beginning with the seminal works by \citet{balcan2006agnostic, dasgupta2007general}, follow a familiar paradigm of disagreement based learning. This involves maintaining a version space of promising hypotheses at each time and constructing a disagreement region of unlabeled examples. For any unlabeled example in the disagreement region, there exists two hypotheses in the version disagreeing on their predictions. An example then chosen for annotation by sampling from a informative sampling distribution computed over the disagreement region. Several approaches have been proposed for computing such sampling distributions, e.g. \citet{jain2019new,katz2020empirical,katz2021improved, huang2015efficient}. As described in Section~\ref{ssec:algorithm}, our main subroutine VReduce is equivalent to fixed-budget one dimensional threshold disagreement learning based on the ACED algorithm of \citet{katz2021improved}. We remark that these algorithms tend to be overly pessimistic in training deep neural nets, and this paper hopes to close this gap.

\textbf{Deep Active Learning under Label Noise}
Label noisy settings has rarely been studied in the deep active learning literature. Related but tangential to our work, several papers have studied to use active learning for cleaning existing noisy labels~\citep{lin2016re,younesian2021qactor}. In this line of work, they assume access to an oracle annotator that will provide clean labels when queried upon. This is fundamentally different from our work, where our annotator may provide noisy labels. Another line of more theoretical active learning research studies active learning with multiple annotators with different qualities~\citep{zhang2015active,chen2022improved}. The primary goal in these work is to identify examples a weak annotator and a strong annotator may disagree, in order to only use the strong annotator on such instances. In our work, we assume access to a single source of annotator that is noisy, which is prevalent in annotation jobs today. Recently, \citet{khosla2022neural} proposed a novel deep active learning algorithm specialized for Heteroskedastic noise, where different ``regions" of examples are subject to different levels of noise. Unlike their work, our work is agnostic to the noise distributions and conduct experiments on uniformly random corrupted labels.
To our knowledge, no deep active learning literature has studied the scenario where both imbalance and label noise present. Yet, this setting is the most prevalent in real-world annotation applications.

\section{Preliminary}
\subsection{Notations} \label{sec:notation}
We study the pool-based active learning problem, where an initial unlabeled set of $N$ examples $X = \{x_1, ..., x_N\}$ are available for annotation.  Their corresponding labels $Y = \{y_1, ..., y_N\}$ are initially unknown. Furthermore, we study the multi-class classification problem, where the space of labels $\mc{Y} := [K]$ is consisted of $K$ classes. Moreover, let $N_1, ..., N_K$ denote the number of examples in $X$ of each class. We define the imbalance ratio as $\gamma = \frac{\min_{k\in[K]} N_k}{\max_{k'\in[K]} N_{k'}}$.

A deep active learning algorithm iteratively chooses batches of examples for annotation. During the $t$-th iteration, the algorithm is given labeled and unlabeled sets of examples, $L_t$ and $U_t$ respectively, where $L_t \cup U_t = X$ and $L_t \cap U_t = \emptyset$. The algorithm then chooses $B$ examples from the unlabeled set $X^{(t)} \subseteq U_t$ and then obtains  their corresponding labels $Y^{(t)}$. The labeled and unlabeled sets are then updated, i.e., $L_{t+1} \leftarrow L_t \cup X^{(t)}$ and $U_{t+1} \leftarrow U_t \backslash X^{(t)}$. Based on new labeled set $L_{t+1}$ and its corresponding labels, a neural network $f_t : X \rightarrow [K]$ is trained to inform the choice for the next iteration. The ultimate goal of deep active learning is to obtain high predictive accuracy for the trained neural network while annotating as few examples as possible.

\subsection{Limitations of Existing Imbalanced Active Learning Algorithms} \label{sec:review}
\begin{figure}
    \begin{subfigure}[t]{.48\textwidth}
        \includegraphics[width=\textwidth]{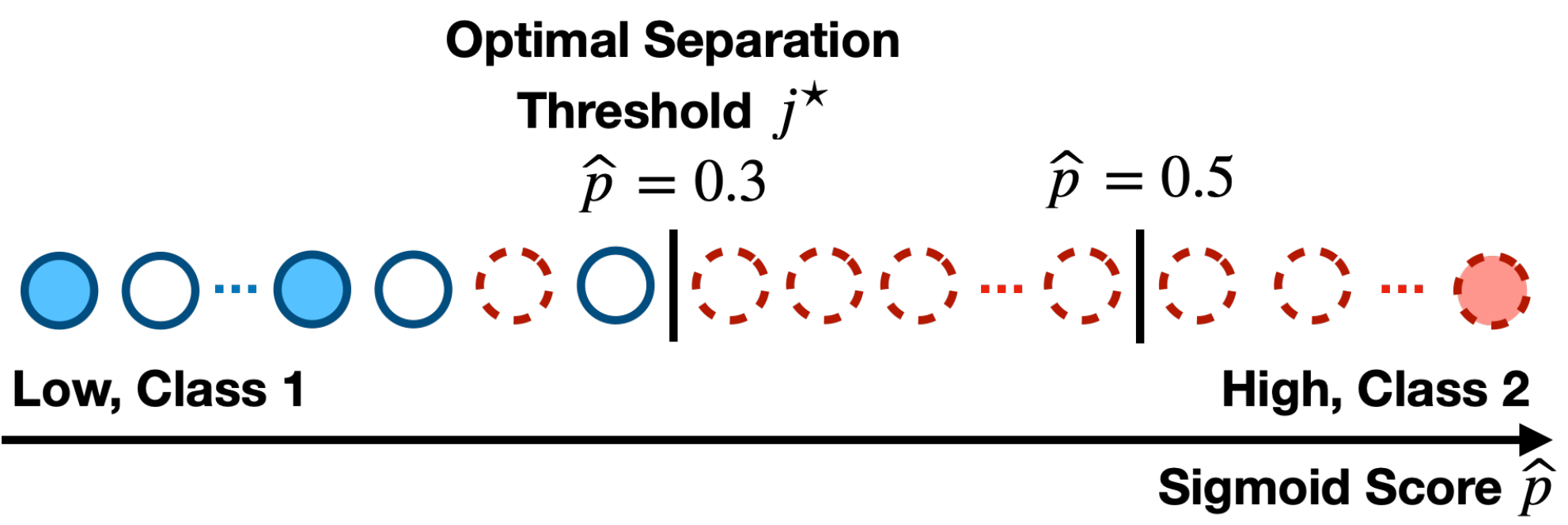}
        \caption{Uncertainty based methods that query around $\widehat{p}=.5$ could annotate examples only in the majority class.}
        \label{fig:galaxy}
    \end{subfigure}
    \hspace{.03\textwidth}
    \begin{subfigure}[t]{.48\textwidth}
        \includegraphics[width=\textwidth]{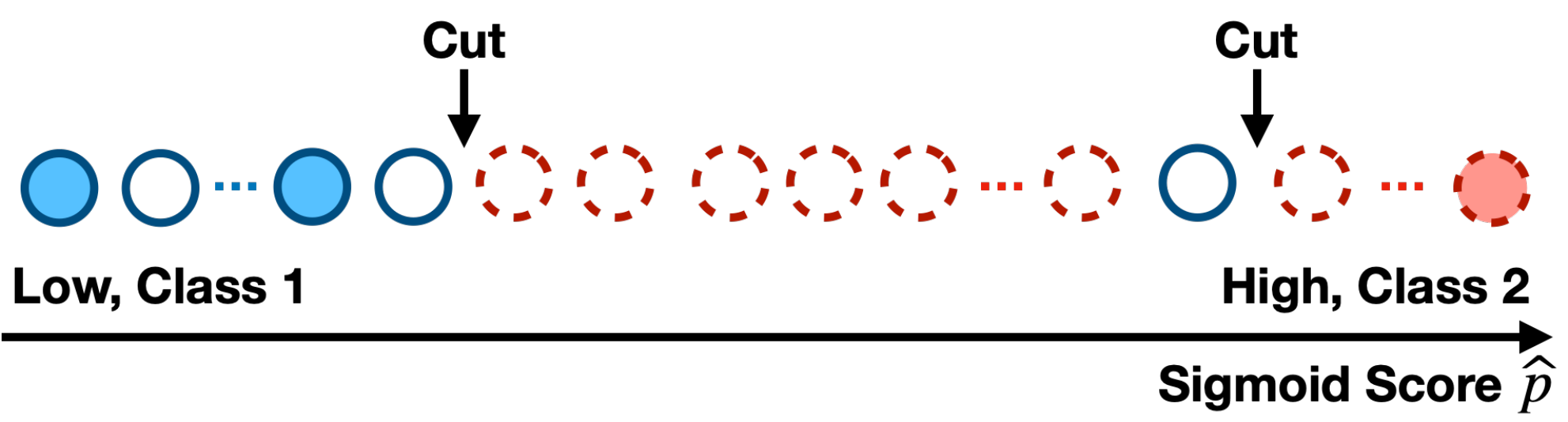}
        \caption{GALAXY spends approximately equal annotation budget around both cuts, while the cut on the right would yield examples mostly in the majority class.}
        \label{fig:galaxy2}
    \end{subfigure}
    \caption{Demonstration of existing imbalance active learning algorithms. Ordered lists of examples are ranked by the predictive sigmoid score $\widehat{p}$. The ground truth label of each example is represented by its border -- solid blue for class $1$ and dotted red for class $2$. Annotated examples are shaded.}
    \vspace{-1.5\intextsep}
\end{figure}

Below we document the several active learning algorithms and how their progressive improvement. At the end, we highlight the shortcomings of the state-of-art algorithm GALAXY~\citep{zhang2022galaxy} and motivate DIRECT's objective of adaptively finding the \emph{optimal separation threshold}. We first consider an imbalanced binary classification case, where $N_1 < N_2$ without loss of generality.

\textbf{Random Sampling.} After annotating a significant number of examples, random sampling would annotate a subset of $X$ with an imbalance ratio close to $\frac{N_1}{N_2}$. This approach suffers from annotating examples that are neither class-balanced nor informative.

\textbf{Uncertainty Sampling.} In the binary classification case, uncertainty sampling methods, such as confidence \citep{settles2009active}, margin \citep{tong2001support,balcan2006agnostic} and entropy \citep{kremer2014active} sampling, simply sort examples based on their predictive sigmoid scores $\widehat{p}$ and annotate examples closest to $.5$ as demonstrated in Figure~\ref{fig:galaxy}. As shown in our results in Figure~\ref{fig:intro} and Section~\ref{sec:experiments}, uncertainty sampling, despite improving over random sampling, significantly underperforms DIRECT and GALAXY and consistently collects less balanced annotations. This shortcoming suggests there are significantly more majority examples than minorities around the decision boundary of $\widehat{p} = .5$.

\textbf{Objective of DIRECT.} To mitigate the above issue with the decision boundary, we propose to identify the \emph{optimal separation threshold}. The threshold best separates the minority and majority classes and approximately equalizes the number of examples from both classes around its vicinity (see Section~\ref{ssec:reduction} for formal definition). We note the optimal separation threshold could be relatively distant from $\widehat{p} = .5$, as shown in Figure~\ref{fig:galaxy}. Our overall objective is to label examples that are \emph{both uncertain and class-balanced}, and can be decomposed into the following two-phased procedure: 
\fbox{
    \parbox{.95\linewidth}{
        \begin{enumerate}
            \item Identify the \emph{optimal separation threshold} $j^\star$ that best separates the minority class from the majority class, as shown in Figure~\ref{fig:galaxy}.
            \item Annotate equal number of examples next to $j^\star$ from both sides.
        \end{enumerate}
    }
}

\textbf{Limitation of GALAXY\citep{zhang2022galaxy}.} As discussed above, the neural network decision boundary $\widehat{p} = .5$ does not necessarily best separate minority and majority class examples. GALAXY draws inspiration from graph-based active learning. It relies on the fact that the best separation threshold must be a cut, namely thresholds with a minority class example to the left and a majority class example to the right (see Figure~\ref{fig:galaxy2}). The algorithm aims to find \emph{all} cuts in the sorted graph as shown in Figure~\ref{fig:galaxy2}. However, GALAXY suffers from three weaknesses:
\begin{enumerate}
    \item During active learning, the neural network is still under training and cannot perfectly separate the two classes of examples yet. Therefore, the sorted graph could have a significant number of cuts. As an example in Figure~\ref{fig:galaxy2}, when annotating around all of such cuts, the algorithm could waste a significant portion of the annotation budget around misclassified outliers, leading to a large number of majority class annotations.

    \item Under label noise, the incorrect annotation could lead to more cuts in the sorted graph, further exacerbating the above issue.

    \item GALAXY finds all cuts through a modified bisection procedure, which only allows for sequential labeling and prevents multiple annotators labeling in parallel.
\end{enumerate}
In this paper, we take a DIRECT approach by identifying only the optimal separation threshold and address all of the shortcomings above.

\section{A Robust Algorithm for Active Learning under Imbalance and Label Noise}
In this section, we formally define the optimal separation threshold and pose the problem of identifying it as an 1-dimensional reduction to the agnostic active learning problem. We then propose an algorithm inspired by the agnostic active learning literature \citep{balcan2006agnostic,dasgupta2007general,hanneke2014theory,katz2021improved}.

\subsection{An 1-D Reduction to Agnostic Active Learning}\label{ssec:reduction}
We start by considering the imbalanced binary classification setting mentioned in Section~\ref{sec:review}. When given a neural network model, we let $\widehat{p}: X \rightarrow [0, 1]$ be the predictive function mapping examples to sigmoid scores. Here, a higher sigmoid score represents a higher confidence of the example being in class $2$. We sort examples by their sigmoid predictive score similar to Section~\ref{sec:review}. Formally, we now define the optimal separation threshold as described in Section~\ref{sec:review}.

\begin{definition} 
    Let $0 = q_{(0)} \leq q_{(1)} \leq \cdots \leq q_{(N)}$, where $\{q_{(i)} \in \R\}_{i=1}^N$ is a sorted permutation of $\{\widehat{p}(x_i)\}_{i=1}^N$. Further we let $\{x_{(i)}\}_{i=1}^N$ and $\{y_{(i)}\}_{i=1}^N$ denote the sorted list's corresponding examples and labels. We define the \emph{optimal separation threshold} as $j^\star \in \{0, 1, ..., N\}$ such that
    \begin{align}\label{eqn:obj}
        j^\star = \argmax_j &\left(\lvert\{y_{(i)} = 1 : i\leq j\}\rvert - \lvert\{y_{(i)} = 2 : i\leq j\}\rvert\right) \nonumber \\
        = \argmax_j &\left( \lvert\{y_{(i)} = 2 : i > j\}\rvert - \lvert\{y_{(i)} = 1 : i > j\}\rvert\right).
    \end{align}
    % \begin{align}
    %     j^\star = \argmax_j &\left(\lvert\{y_{(i)} = 1 : 1\leq i\leq j\}\rvert \right.\nonumber\\ 
    %     &\left. + \lvert\{y_{(i)} = 2 : j < i\leq N\}\rvert\right). \label{eqn:obj}
    % \end{align}
\end{definition}
In other words, on either side of $j^{\ast}$, it has the largest discrepancy in the number of examples between the two classes. This captures the intuition of Figure~\ref{fig:galaxy} --- our goal is to find a threshold that best separates one class from the other.
We quickly remark that ties are broken by choosing the largest $j^\star$ that attains the argmax if class $1$ is the minority class and the lowest $j^\star$ otherwise.

\begin{figure}[t]
% \vspace{-2\intextsep}
\begin{center}
    \includegraphics[width=.4\textwidth]{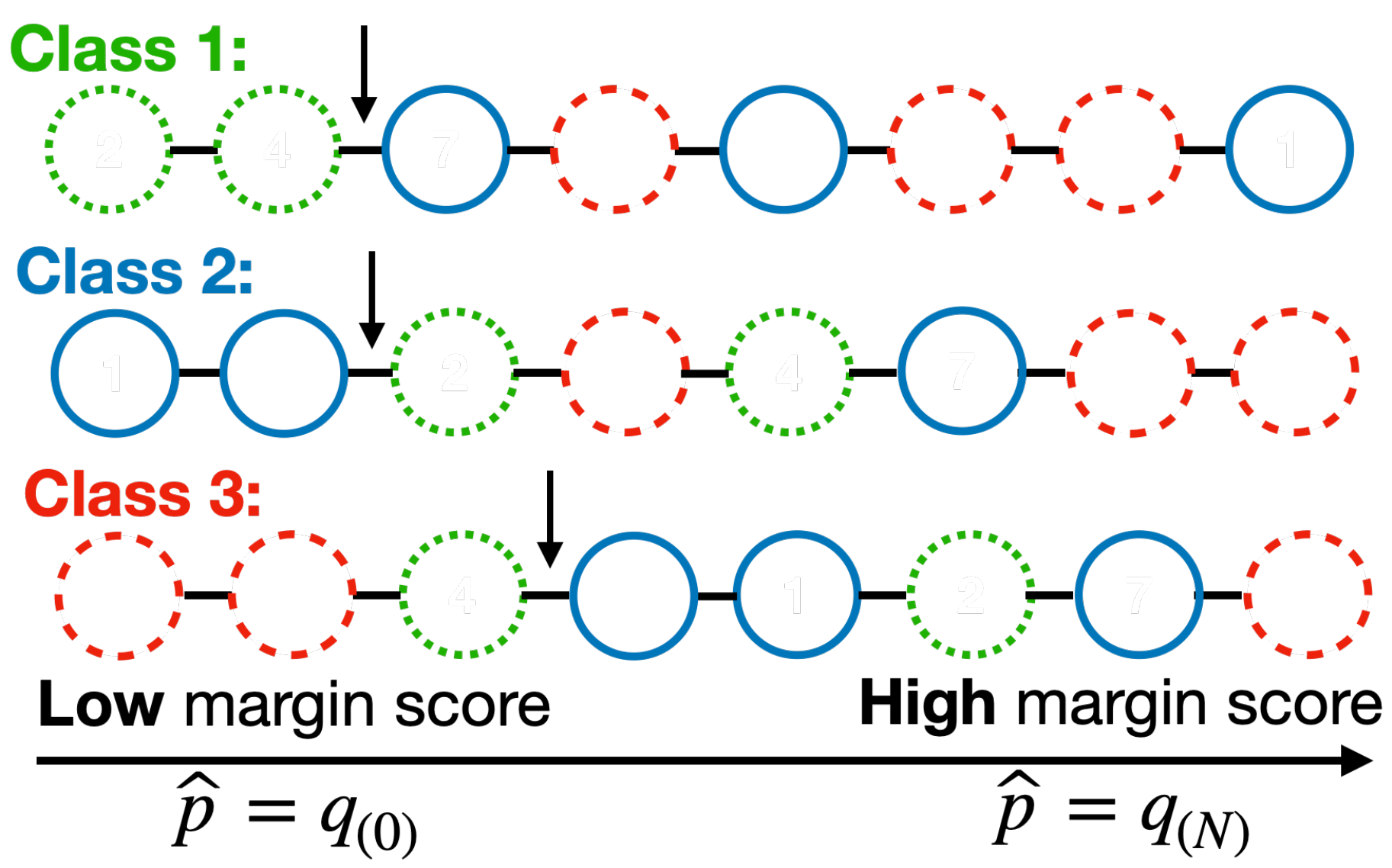}
\end{center}
\vspace{-0.5\intextsep}
\caption{Visualization of multi-class classification. For each class, we formulate the problem as a one-vs-rest binary classification problem by sorting examples based on margin scores. The black arrows indicates the optimal separation thresholds for each class.}
\vspace{-\intextsep}
\label{fig:multiclass}
\end{figure}

\noindent\textbf{1D Reduction.} We now provide a reduction of finding $j^\star$ to an 1-dimensional agnostic active learning problem. We define the hypothesis class $\mc{H} = \{h_0, h_1, ..., h_N\}$ where each hypothesis $h_j$ is defined as $h_j(q) = \begin{cases}
    1 & \text{if} \, q \leq q_{(j)} \\
    2 & \text{if} \, q > q_{(j)}
\end{cases}$.
Here, $q_{(0)} = 0$ defines the hypothesis $h_0$ that predicts class $2$ at all times. The empirical zero-one loss for each hypothesis is then defined as $\mc{L}(h_j) = \sum_{i=1}^N \1\{h_j(q_{(i)}) \neq y_{(i)}\}$. In Appendix~\ref{apx:proof}, we show that optimizing for the zero-one loss $\argmin_{0\leq j\leq N} \mc{L}(h_j)$ is equivalent to \eqref{eqn:obj}. Namely, with ties broken similar to above, $j^\star = \argmin_{0\leq j\leq N} \mc{L}(h_j)$.

\noindent\textbf{Multi-Class Classification.} To generalize the above problem formulation to multi-class classification, we follow a similar strategy to \citet{zhang2022galaxy}. As shown in Figure~\ref{fig:multiclass}, for each class $k$, we can view the problem of class-$k$~v.s.~others as a binary classification problem. The goal therefore becomes finding all $K$ optimal separation thresholds, which is equivalent with solving $K$ 1-D agnostic active learning problems. Moreover, let $\widetilde{p}: X \rightarrow \Delta^{(K-1)}$ denote the neural network prediction function, mapping examples to softmax scores. For each class $k$, we use the margin scores $\widehat{p}_i^k := [\widetilde{p}(x_i)]_k - \max_{k'}[\widetilde{p}(x_i)]_{k'}$ to sort the examples and break ties by their corresponding confidence scores $[\widetilde{p}(x_i)]_k$. Formally,
\begin{align} \label{eqn:sorting}
&\hspace{-.8em}\left(q_{(1)}^k \leq \mkern-2mu{\cdots} \leq q_{(N)}^k\mkern-1mu{:} \text{ sorted permutation of\,} \{\widehat{p}_i^k\}_{i=1}^N\right) \mkern-1mu{\land} \nonumber\\
& \left(q_{(i)}^k \mkern-1mu{=} q_{(i+1)}^k \Rightarrow [\widetilde{p}(x_i)]_k \geq [\widetilde{p}(x_{i+1})]_k\right).\hspace{-.2em}
\end{align}
Note that sorting by margin scores is equivalent to sorting by sigmoid scores for binary classification.

\subsection{One-Dimensional Agnostic Active Learning}

The key insight of our approach is to leverage the well-established theory of agnostic active learning for threshold classifiers, which provides robust guarantees even under label noise. This allows us to robustly identify the optimal separation threshold $j^\star$ from the 1-D reduction above.

\textbf{Problem Formulation and Noise Handling.} In the agnostic setting, we have a sorted sequence $q_{(1)} \leq q_{(2)} \leq \cdots \leq q_{(N)}$ with corresponding (possibly noisy) binary labels $y_{(1)}, y_{(2)}, \ldots, y_{(N)}$. The fundamental challenge is that no hypothesis $h_j \in \mc{H}$ may achieve zero empirical loss $\mc{L}(h_j)$ due to label noise or model misspecification.

To formalize how our algorithm handles label noise, we consider the underlying conditional probability function $P(y_i|x_i)$ for each data example $x_i$. Through the dimensionality reduction $x_i \rightarrow q_i$ (where $q_i$ is the real-valued sigmoid/margin score), we obtain an ordered set of 1-dimensional features $\{q_i\}_{i=1}^N$. This mapping induces a distribution $P(y_i|q_i)$ over the reduced space, which naturally encodes any label noise present in the annotations.

Our objective is to find the threshold classifier $h_{j^\star}$ from the set of 1-dimensional threshold classifiers $\{h_j\}$ that minimizes the probability of error with respect to $P(y_i|q_i)$. Crucially, our agnostic approach makes no assumptions about the form of $P(y_i|x_i)$ or the induced $P(y_i|q_i)$, allowing the algorithm to handle arbitrary noise models without requiring prior knowledge of the noise distribution.

\textbf{Version Space Reduction.} The VReduce algorithm (Algorithm~\ref{alg:vred}) implements a version space approach that maintains an interval $[I, J]$ representing plausible optimal thresholds. The key principle is that if we observe labeled examples $x_{(i)}$ with $i \leq I$ mostly belong to class $1$, and examples $x_{(j)}$ with $j \geq J$ mostly belong to class $2$, then the optimal threshold $j^\star$ with high likelihood will lie within $[I, J]$.

The algorithm proceeds iteratively by: (1) maintaining and updating the version space interval $[I, J]$ based on observed labels, (2) sampling unlabeled examples uniformly within this interval to maximize information gain, (3) shrinking the version space based on empirical loss estimates that account for the noisy observations, and (4) repeating until the labeling budget is exhausted.

\begin{algorithm}[t]
\begin{algorithmic}
\STATE\textbf{Input: } Labeled set $L$, budget $b$, class of interest $k$, parallel batch size $B_{\text{parallel}}$, sorted examples $\{x_{(i)}^k, y_{(i)}^k, q_{(i)}^k\}_{i=1}^N$ (note $y_{(i)}^k$ of unlabeled examples are hidden to the learner).

\STATE\textbf{Initialize: } Version space $[I, J]$ as the shortest interval such that: $\forall i \leq I$ with $x_{(i)} \in L$: $y_{(i)} = k$, and $\forall j \geq J$ with $x_{(j)} \in L$: $y_{(j)} \neq k$.

\STATE Number of iterations $m \leftarrow b/B_{\text{parallel}}$. Shrinking factor $c \leftarrow (J - I)^{1/m}$.

\FOR{$t = 1, ..., m$}
    \STATE Sample $B_{\text{parallel}}$ unlabeled examples uniformly from $\{x_{(I)}^k, \ldots, x_{(J)}^k\}$ and add to $L$.

    \STATE Compute empirical loss: \\$\widehat{\mc{L}}^{k}(s) = \sum_{r \leq s: x_{(r)} \in L} \1\{y_{(r)} \neq k\} + \sum_{r > s: x_{(r)} \in L} \1\{y_{(r)} = k\}$.

    \STATE Update version space:\\ $[I, J] \leftarrow \argmin_{[i, j]: j - i = (J-I)/c} \max\{\widehat{\mc{L}}^{k}(i), \widehat{\mc{L}}^{k}(j)\}$.
\ENDFOR

\textbf{Return: } Updated labeled set $L$.
\end{algorithmic}
\caption{VReduce: Version Space Reduction for Threshold Learning}
\label{alg:vred}
\end{algorithm}

\textbf{Theoretical Guarantees.} The VReduce algorithm inherits robust theoretical properties from the agnostic active learning literature, achieving near-minimax and instance-optimal sample complexity bounds \citep{dasgupta2007general,hanneke2014theory}. Unlike bisection-based methods that fail under label corruption, the disagreement-based framework \citep{balcan2006agnostic} provides natural robustness with high probability guarantees regardless of the underlying data distribution or noise model. Building on the ACED framework \citep{katz2021improved}, our algorithm extends these classical guarantees to practical batch settings through the $B_{\text{parallel}}$ parameter while preserving all theoretical properties, making it suitable for real-world annotation scenarios with multiple parallel annotators.

\setlength{\textfloatsep}{\intextsep}
\subsection{Algorithm}\label{ssec:algorithm}
We are now ready to state our algorithm DIRECT as shown in Algorithm~\ref{alg:direct}. Each round of DIRECT follows a two-phased procedure, where the first phase aims to identify the optimal separation threshold for each class using the agnostic active learning approach described above. The second phase then annotates examples closest to the estimated optimal separation thresholds for each class. We spend half each round's budget for both phases.

\begin{algorithm}[t]
\begin{algorithmic}

\STATE\textbf{Input: } Pool $X$, \#Rounds $T$, retraining batch size $B_{\text{train}}$, number of parallel annotations $B_{\text{parallel}}$.

\STATE\textbf{Initialize: } Uniformly sample $B$ elements from $X$ to form $L_0$. Let $U_0 \leftarrow X\backslash L_0$.

\FOR{$t = 1, ..., T - 1$}
    \STATE Train neural network on $L_{t-1}$ and obtain $f_{t-1}$.

    \STATE \underline{\textbf{Find optimal separation thresholds}}

    \STATE Initialize labeled set $L_t \leftarrow L_{t-1}$ and budget per class $b \leftarrow B_{\text{train}} / 2K$.

    \FOR{$k$ in $\text{RandPerm}(\{1, ..., K\})$}
        \STATE Sort margin scores $0 = q_{(0)}^k \leq q_{(1)}^k \leq \cdots \leq q_{(N)}^k$ based on \eqref{eqn:sorting}.

        \STATE Let $x_{(i)}^k, y_{(i)}^k$ denote the example and label corresponding to $q_{(i)}^k$.
    
        \STATE Identify threshold for class $k$: $L_t \leftarrow \text{VReduce}(L_t, b, k, B_{\text{parallel}}, \{x^k_{(i)}, y^k_{(i)}, q^k_{(i)}\}_{i=1}^N)$.
    \ENDFOR

    \STATE \underline{\textbf{Annotate examples around the identified threshold}}

    \STATE Compute budget per class $b \leftarrow (B_{\text{train}} - |L_t|) / K$.
    
    \FOR{$k$ in $\text{RandPerm}(\{1, ..., K\})$}
        \STATE Estimate separation threshold (break ties by choosing the index closest to $\frac{N}{2}$):\\
        $\quad\widehat{j}^k \leftarrow \argmax_j(|\{y_{(i)} = k: x_{(i)} \in L_t \,\text{and}\, i\leq j\}| - |\{y_{(i)} \neq k: x_{(i)} \in L_t \,\text{and}\, i\leq j\}|)$.

        Annotate $b$ unlabeled examples with sorted indices closest to $\widehat{j}^k$ and insert to $L_t$.
    \ENDFOR
\ENDFOR
\STATE \textbf{Return: } Train final classifier $f_T$ based on $L_T$.
\end{algorithmic}
\caption{DIRECT: DImension REduction for aCTive Learning under Imbalance and Label Noise}
\label{alg:direct}
\end{algorithm}

\begin{figure*}[t]
    \centering
    \begin{subfigure}[t]{.32\textwidth}
        \centering
        \includegraphics[width=\textwidth]{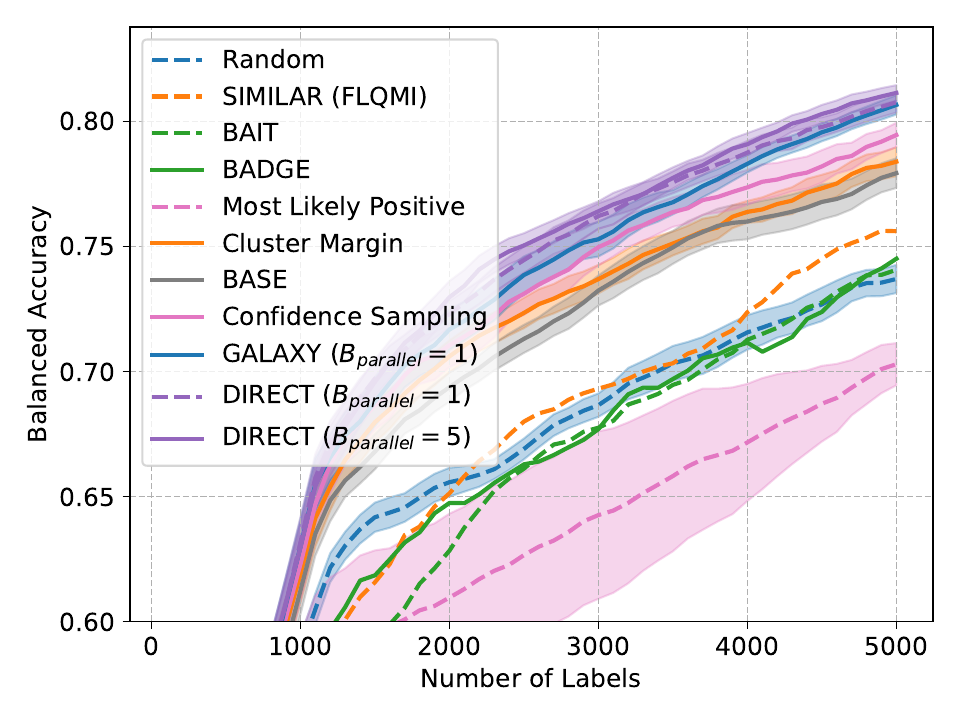}
        \caption{Imbalanced CIFAR-10, three classes}
    \end{subfigure}
    \begin{subfigure}[t]{.32\textwidth}
        \centering
        \includegraphics[width=\textwidth]{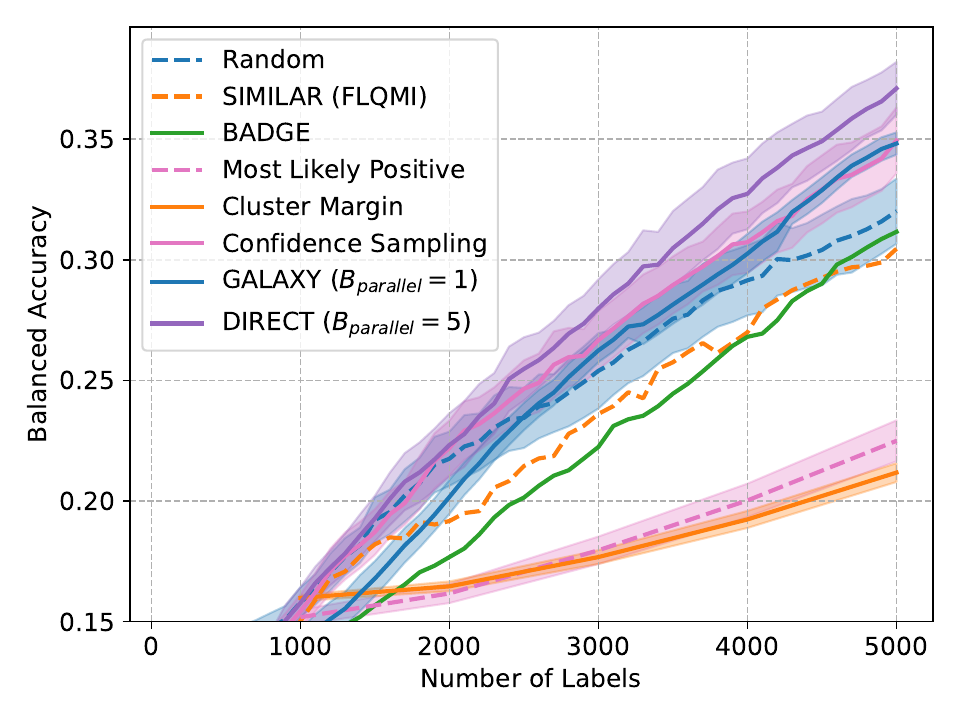}
        \caption{Imbalanced CIFAR-100, ten classes}
    \end{subfigure}
    \begin{subfigure}[t]{.32\textwidth}
        \centering
        \includegraphics[width=\textwidth]{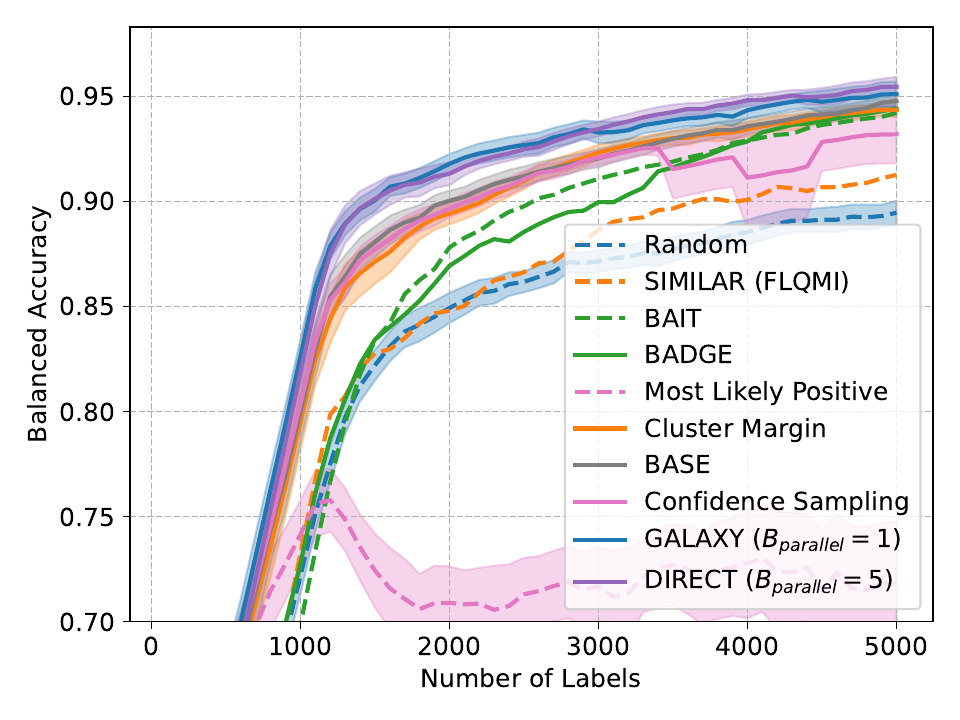}
        \caption{Imbalanced SVHN, two classes}
    \end{subfigure}
    \begin{subfigure}[t]{.32\textwidth}
        \centering
        \includegraphics[width=\textwidth]{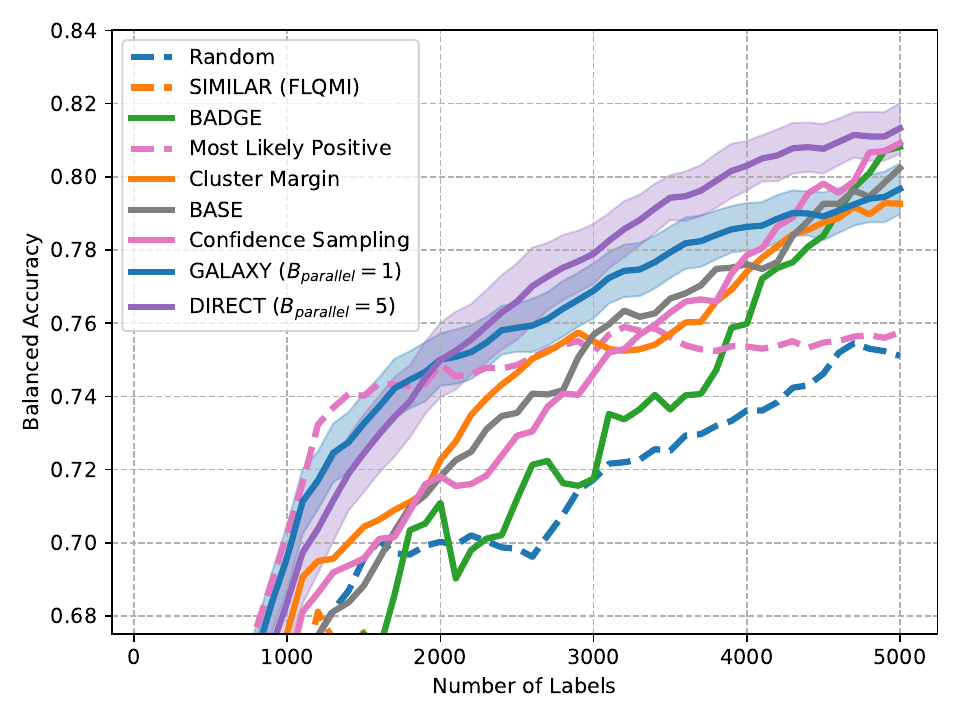}
        \caption{PathMNIST, two classes}
    \end{subfigure}
    \begin{subfigure}[t]{.32\textwidth}
        \centering
        \includegraphics[width=\textwidth]{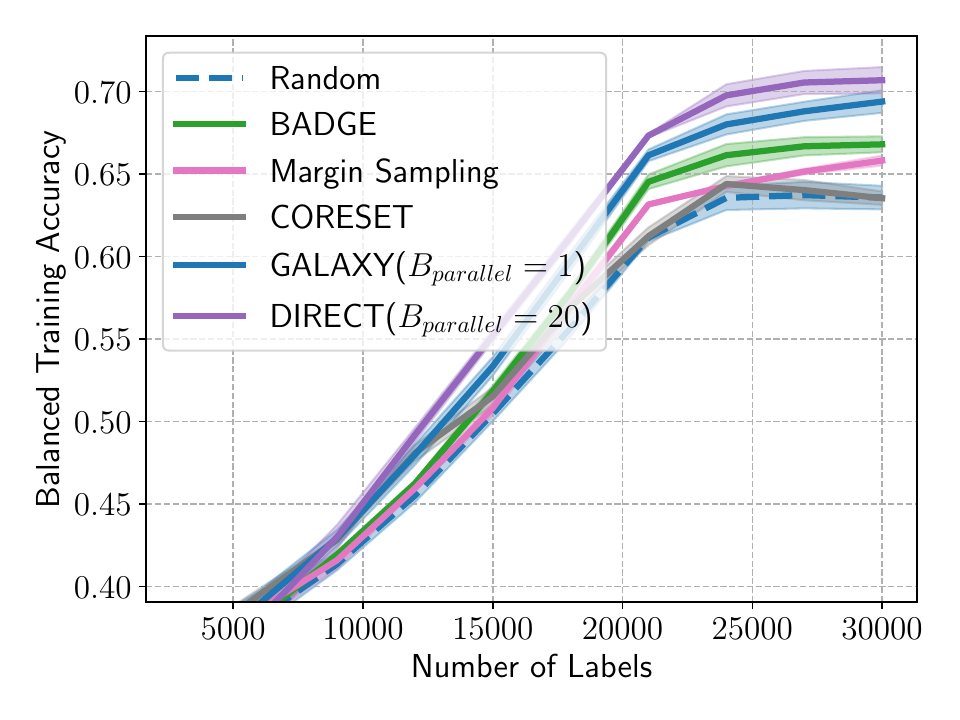}
        \caption{FMoW, Balanced Pool Accuracy}
    \end{subfigure}
    \begin{subfigure}[t]{.32\textwidth}
        \centering
        \includegraphics[width=\textwidth]{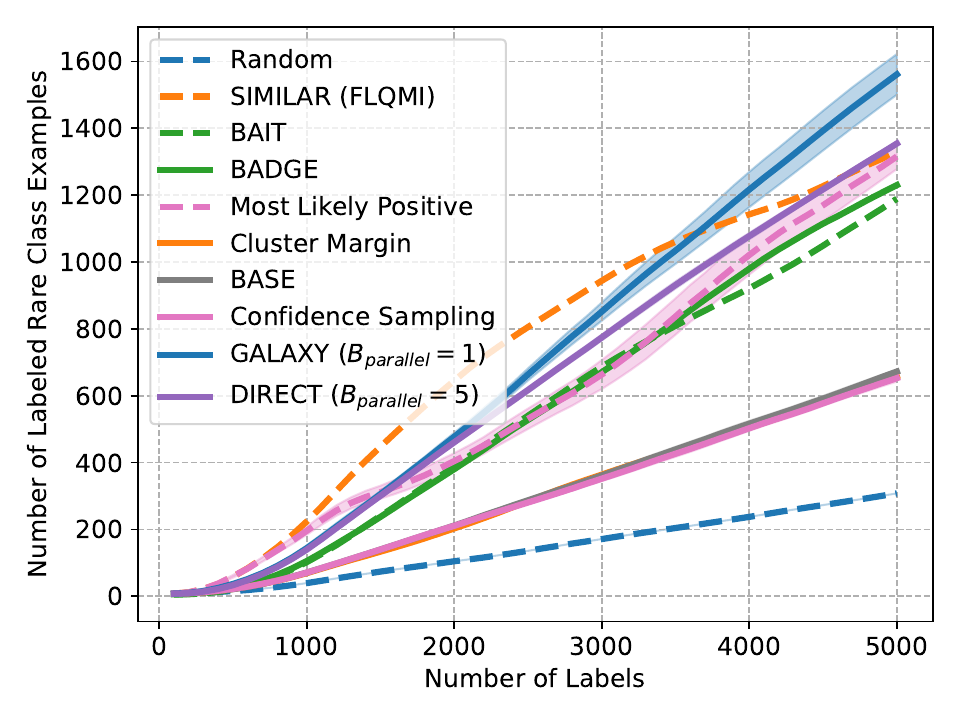}
        \caption{Imbalanced SVHN, two classes, \# of minority labeled examples}
    \end{subfigure}
    \caption{Performance comparison of DIRECT against other baselines algorithms in the noiseless but imbalanced setting. (a)-(d) are balanced accuracy on ResNet-18 experiments while (e) shows experiment under the LabelBench framework. $B_\text{parallel}$ indicates the number of parallel annotations as mentioned in Section~\ref{ssec:algorithm}. $B_\text{parallel}=1$ is equivalent with sychornous labeling. Results are averaged over four trials and the shaded areas represent standard errors around the mean.}
    \label{fig:noiseless_results}
    \vspace{-\intextsep}
\end{figure*}

During the first phase, to identify the optimal separation threshold for all classes, we loop over each class $k$ and run the agnostic active learning procedure VReduce for the corresponding 1-D class-k~v.s.~rest reduction. The second phase of DIRECT simply annotates examples closest to each optimal separation threshold, aiming to annotate a class-balanced and uncertain examples.

To address batch labeling, we let $B_{\text{train}}$ denote the number of examples the algorithm collects before the neural network is retrained. In practice, this number is usually determined by the constraints of computational training cost. On the other hand, we let $B_{\text{parallel}}$ denote the number of examples annotated in parallel. We note that, in practice, the number of examples collected before retraining is usually far greater than the number of annotators annotating in parallel, i.e., $B_{\text{parallel}} \ll B_{\text{train}}$. Lastly, as will be discussed in Section~\ref{sec:future}, our algorithm can also be modified for asynchronous labeling.

\textbf{Theoretical Comparison with GALAXY.}
As mentioned in Section~\ref{sec:review}, GALAXY's graph-based approach aims to identify all \emph{cuts} and sample examples around all of them equally. On the other hand, DIRECT aims to identify only the separation threshold and sample around it, which is superior as we have argued before and shown in our results. We now present a more theoretical comparison. As we show in Appendix~\ref{apx:proof}, the graph-based approach in GALAXY will identify and annotate around at least one more cut in addition to the optimal separation threshold, with probability at least $1 - \exp(-b\log(\frac{1}{1-\eta})/ 2)$. Here, $b$ is the budget of a single round of annotation and $\eta$ is the label noise ratio (see Appendix~\ref{apx:proof} for more details). This implies, when the budget $b$ is large, GALAXY will likely annotate around unnecessary cuts. This is in contrast with the agnostic active learning approach we take in DIRECT, where as shown by \citet{katz2021improved}, the probability of misidentifying the optimal separation threshold decays exponentially w.r.t. budget $b$. In other words, with a large budget $b$, with high likelihood, DIRECT will focus its annotation around the optimal separation threshold. Lastly, time complexity analysis in Appendix~\ref{apx:complexity} shows DIRECT's superior speed compared to BADGE and GALAXY.

\section{Experiments} \label{sec:experiments}
We conduct experiments under two primary setups:
\begin{enumerate}
    \item Supervised fine-tuning of ResNet-18 on imbalanced datasets similar to~\citet{zhang2022galaxy}.
    \item Fine-tuning large pretrained model (CLIP ViT-B32) with semi-supervised training strategies under the LabelBench framework~\citep{zhang2024labelbench}.
\end{enumerate}
For both evaluation setups, we first evaluate the performance of DIRECT under the noiseless setting in Section~\ref{sec:noiseless}, showing its superior label-efficiency and ability to accommodate batch labeling. In Section~\ref{sec:labelnoise}, we evaluate deep active learning algorithms under a novel setting with both class imbalance and noisy labels. Under this setting, we also include an ablation study of the performance of DIRECT on various levels of label noises. While we highlight many results in this section, see Appendix~\ref{apx:all_results} for complete results.  Our implementation is publicly available at: \url{https://github.com/EfficientTraining/LabelBench/blob/main/LabelBench/strategy/strategy_impl/direct.py} 

\textbf{Experiment Setups.} Our experiments utilize $14$ imbalanced datasets derived from popular computer vision datasets. For the ResNet experiments, we utilize imbalanced and/or long-tail versions of CIFAR-10, CIFAR-100~\citep{krizhevsky2009learning}, SVHN~\citep{netzer2011reading} and PathMNIST~\citep{yang2021medmnist} datasets. For the LabelBench experiments, we utilize the FMoW~\citep{christie2018functional}, iWildcam~\citep{beery2021iwildcam}, iNaturalist~\citep{van2018inaturalist} and ImageNet-LT~\citep{deng2009imagenet} datasets. We refer the readers to Appendix~\ref{sec:setup} for more details on our experiment setups.

\begin{figure*}[t]
    \centering
    \begin{subfigure}[t]{.32\textwidth}
        \centering
        \includegraphics[width=\textwidth]{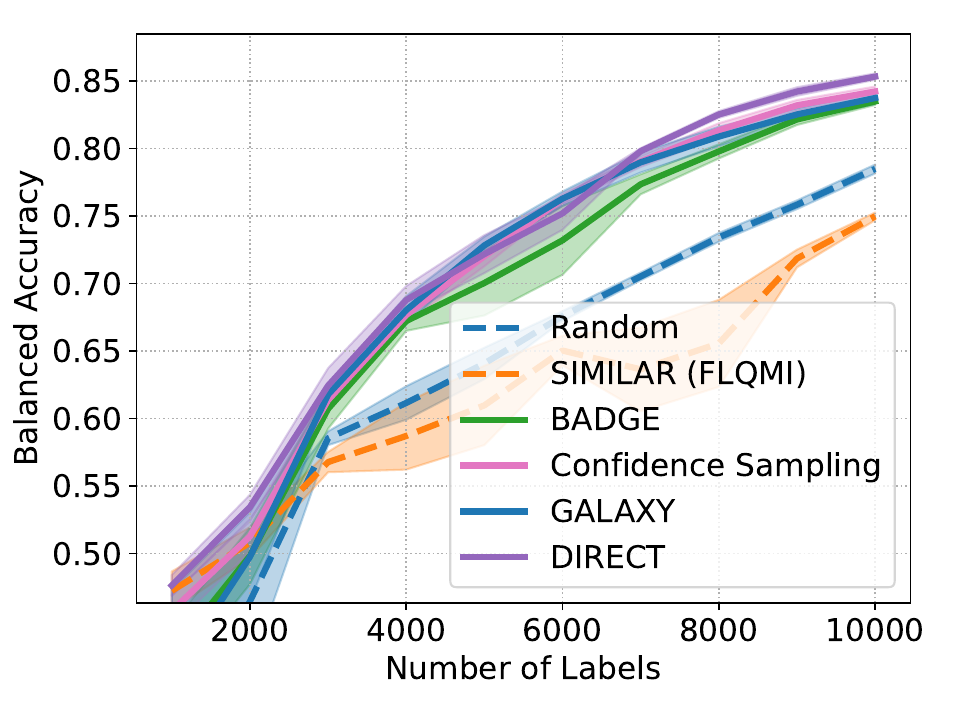}
        \caption{CIFAR-10LT, 10 classes, 15\% label noise}
    \end{subfigure}
    \begin{subfigure}[t]{.32\textwidth}
        \centering
        \includegraphics[width=\textwidth]{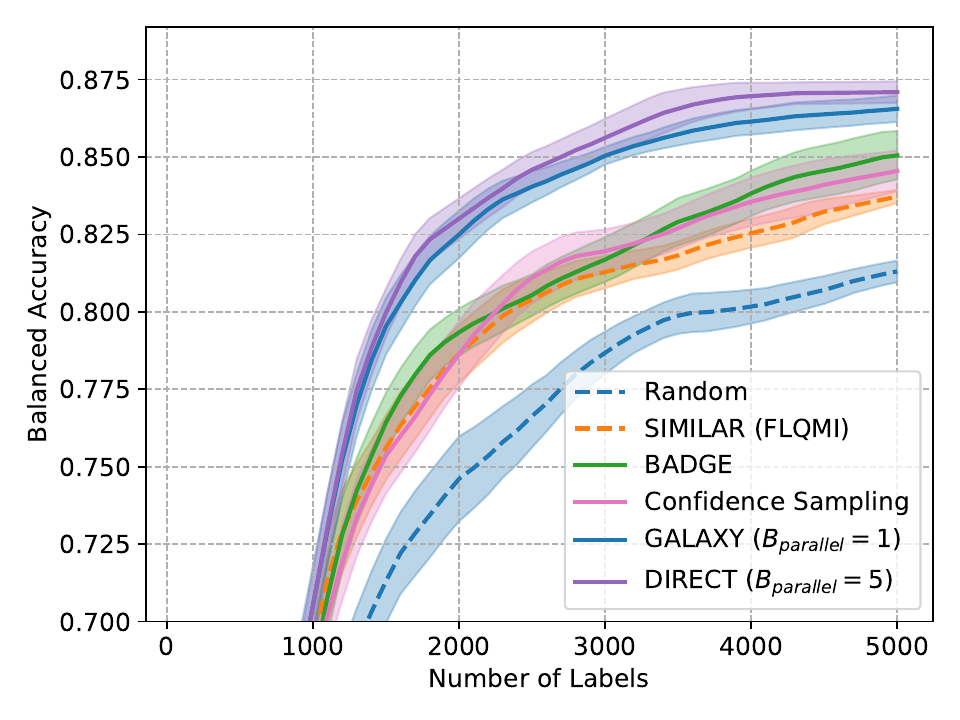}
        \caption{Imbalanced SVHN, two classes, 10\% label noise}
    \end{subfigure}
    \begin{subfigure}[t]{.32\textwidth}
        \centering
        \includegraphics[width=\textwidth]{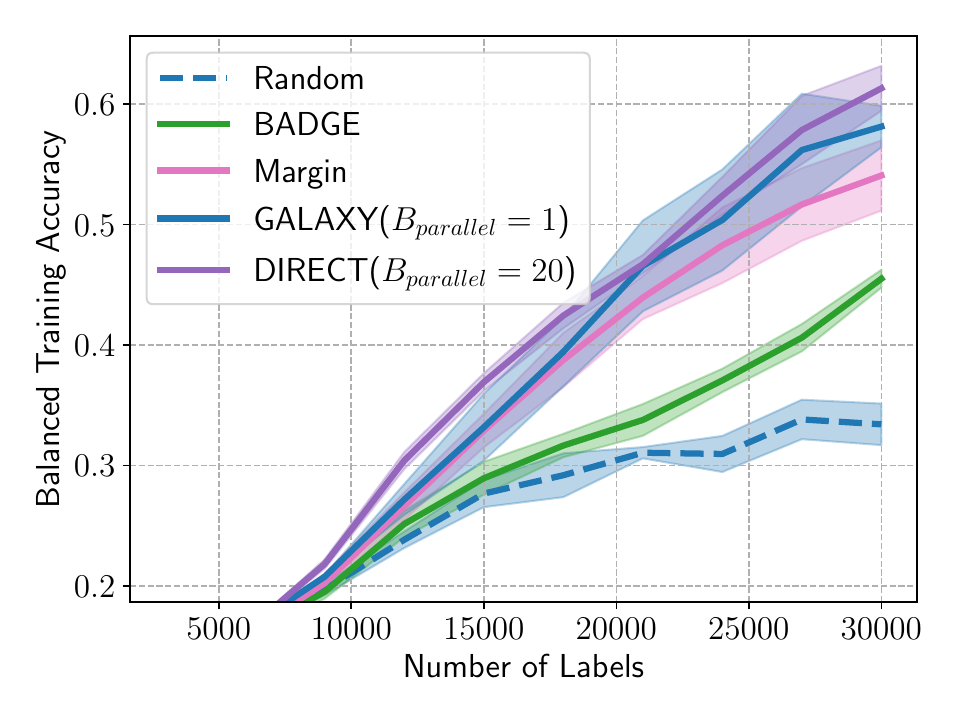}
        \caption{iWildcam Balanced Pool Accuracy, $10\%$ label noise}
    \end{subfigure}
    \caption{Performance of DIRECT against baseline algorithms under label noise.  (a)-(b) are balanced accuracy on ResNet-18 experiments while (c) shows results under the LabelBench framework. Results are averaged over four trials  and the shaded areas represent standard errors around the mean.}
    \label{fig:noisy_results}
    \vspace{-\intextsep}
\end{figure*}
\begin{figure*}[t]
    \centering
    \begin{subfigure}[t]{.32\textwidth}
        \centering
        \includegraphics[width=\textwidth]{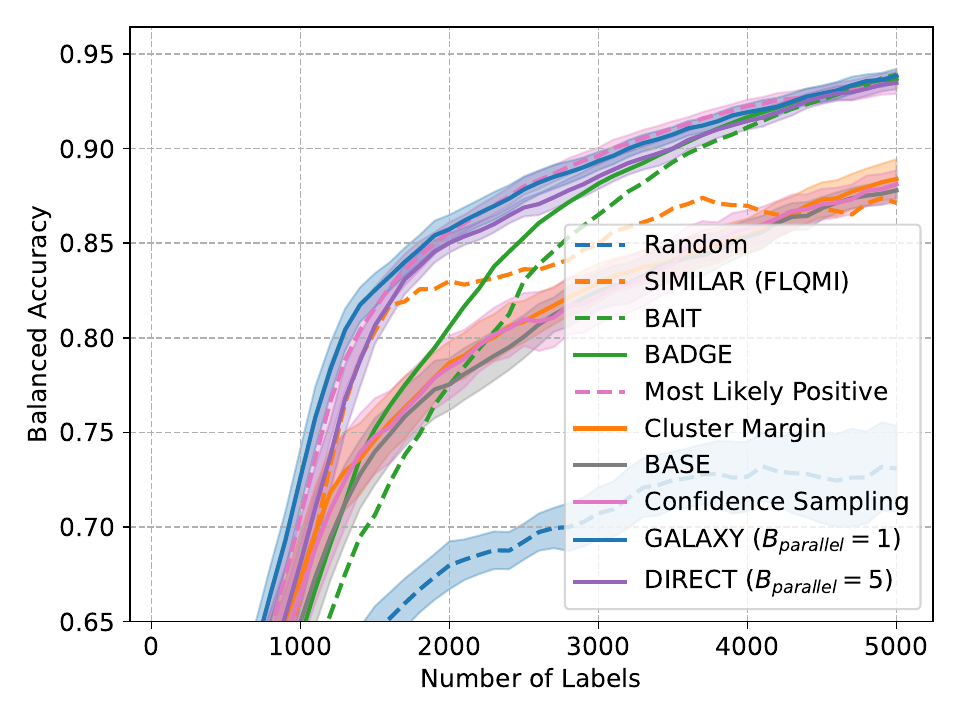}
        \caption{Imbalanced CIFAR-100, two classes, no label noise}
    \end{subfigure}
    \begin{subfigure}[t]{.32\textwidth}
        \centering
        \includegraphics[width=\textwidth]{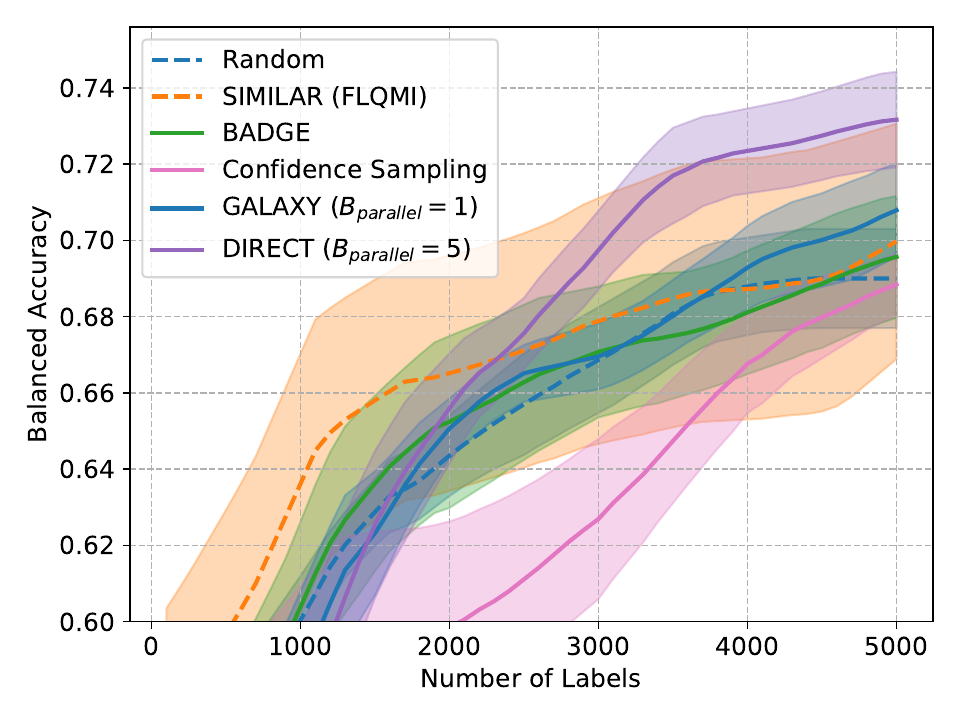}
        \caption{Imbalanced CIFAR-100, two classes, 10\% label noise}
    \end{subfigure}
    \begin{subfigure}[t]{.32\textwidth}
        \centering
        \includegraphics[width=\textwidth]{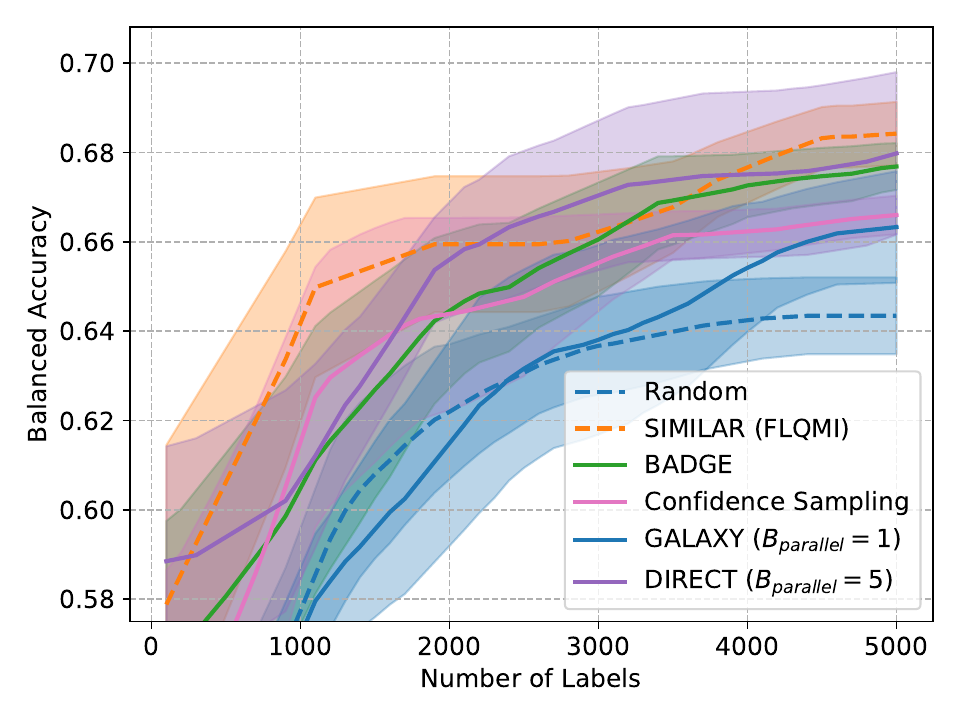}
        \caption{Imbalanced CIFAR-100, two classes, 15\% label noise}
    \end{subfigure}
    \caption{Performance of DIRECT against baseline algorithms under different levels of label noise. Results are averaged over four trials and the shaded areas represent standard errors around the mean.}
    \label{fig:noise_level}
    % \vspace{-\intextsep}
\end{figure*}

\subsection{Experiments under Imbalance, without Label Noise} \label{sec:noiseless}
For the noiseless experiment on ResNet-18, we compare against nine baselines: GALAXY~\citep{zhang2022galaxy}, SIMILAR~\citep{kothawade2021similar}, BADGE~\citep{ash2019deep}, BASE~\citep{emam2021active}, BAIT~\citep{ash2021gone}, Cluster Margin~\citep{citovsky2021batch}, Confidence Sampling~\citep{settles2009active}, Most Likely Positive~\citep{jiang2018efficient,warmuth2001active,warmuth2003active} and Random Sampling. We briefly distinguish the algorithms into two categories. In particular, both SIMILAR and Most Likely Positive annotate examples that are similar to existing labeled minority examples, thus can significantly annotate a large quantity of minority examples. The rest of the algorithms primarily optimizes for different notions of informativeness such as diversity and uncertainty.

\begin{figure*}[ht!]
    \centering
    \includegraphics[width=\linewidth,trim={0 0 0 0.85cm},clip]{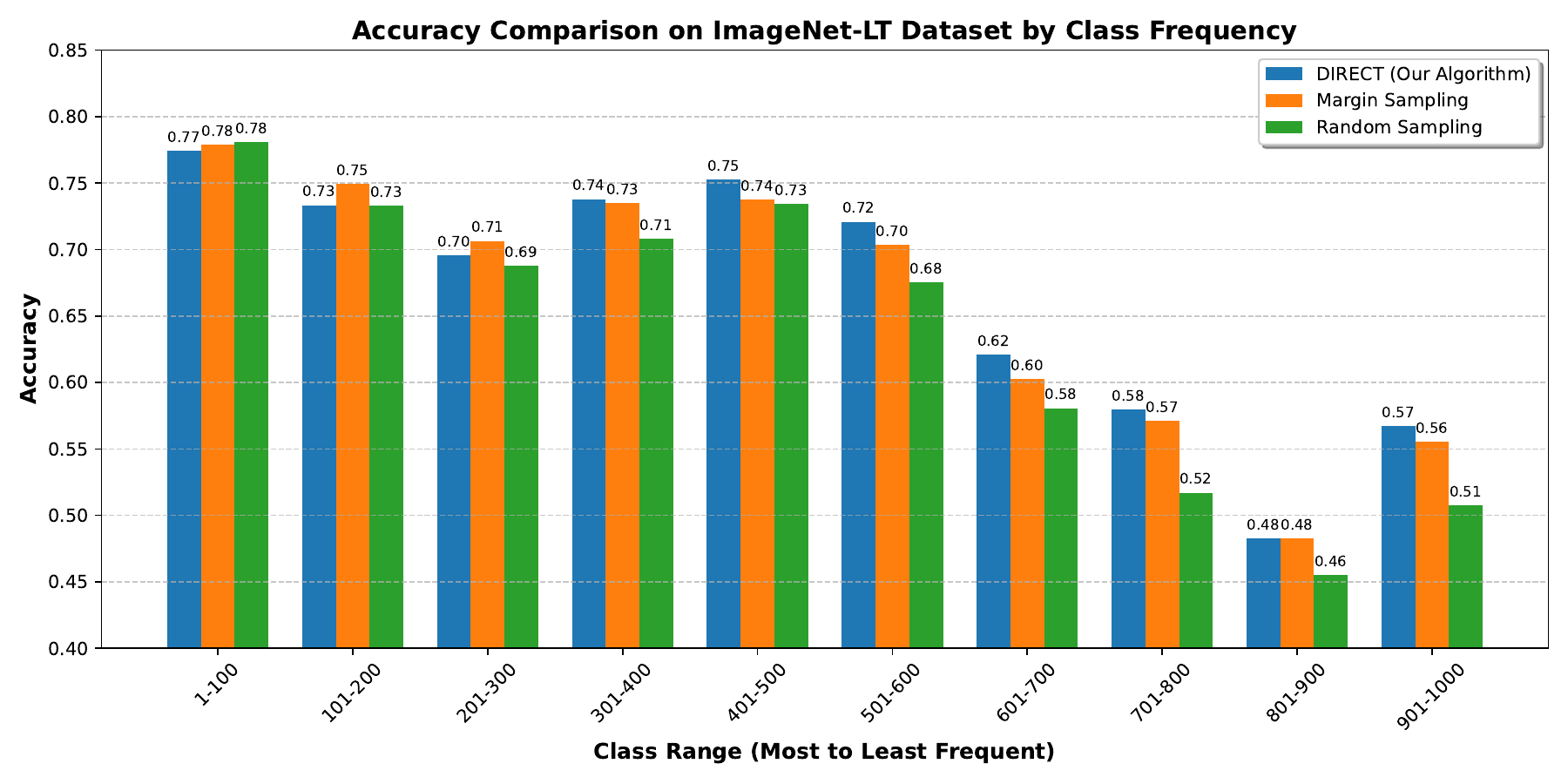}
    \caption{Performance of different active learning algorithm across different classes for the ImageNet-LT dataset.}
    \label{fig:perclass}
    \vspace{-\intextsep}
\end{figure*}

For the LabelBench experiments, due to the large dataset and model embedding sizes, we choose a subset of the algorithms that are computationally efficient and among top performers in the ResNet-18 results, including BADGE, Margin Sampling, CORESET and GALAXY. As highlighted in Figures~\ref{fig:intro}(a) and~\ref{fig:noiseless_results}(a)-(d), DIRECT consistently and significantly outperforms existing algorithms on the ResNet-18 experiments. In Figures~\ref{fig:intro}(c) and~\ref{fig:noiseless_results}(e), we demonstrate the increased label-efficiency is also consistently shown in the LabelBench experiments. Compared to random sampling, DIRECT can save more than 80\% of the annotation cost on imbalanced SVHN experiment of Figure~\ref{fig:noiseless_results}(c). In terms of class-balancedness, we consistently observe that both Most Likely Positive and SIMILAR annotating greater number of minority class examples, but significantly underperforms in terms of balanced accuracy (an example showin in Figure~\ref{fig:noiseless_results}(f)). While \citet{zhang2022galaxy} has already observed this phenomenon, we can further see that DIRECT collects slightly less minority class examples than GALAXY, but outperforms in terms of balanced accuracy. While it is crucial to optimize class-balancedness for better model performance, we see that both extremes of annotating too few and too many minority examples could lead to worse generalization performances. When too few examples are from minority class, the performance of the minority classes could be significantly hindered. When optimized to annotate as many examples from minority class as possible, the algorithm has to tradeoff annotating informative examples to examples it is more certain to be in the minority class. Together, this suggests an intricate balance between the two objectives, generalization performance and class-balancedness.

We would also like to highlight the ability to handle batch labeling. Across our experiments, we see DIRECT outperforms with different amounts of parallel annotation ($B_\text{parallel} = 1, 5$ and $20$), indicating its general effectiveness. This is in comparison to the synchoronous nature of GALAXY, where it is always using $B_\text{parallel} = 1$. On Figures~\ref{fig:intro}(a) and~\ref{fig:noisy_results}(a), we see that DIRECT outperforms GALAXY with synchronous labeling. Furthermore, in these experiments we also see using $B_\text{parallel} = 5$ only affects algorithm performances minimally for DIRECT.

\subsection{Experiments under Imbalance and Label Noise} \label{sec:labelnoise}
We conduct novel sets of experiments under both class imbalance and label noise. Here, for both ResNet-18 and LabelBench experiments, we evaluate against all of the algorithms that performed well under the imbalance but noiseless setting above. For all of our experiments, we introduce a fixed percentage of label noise, where the given fraction of the examples' labels are corrupted to a different class uniformly at random. For most of our experiments with $10\%$ label noise shown in Figure~\ref{fig:noisy_results}, we observe again that DIRECT consistently improves over all baselines including GALAXY. The results are consistent on ResNet-18 and LabelBench setups, and with different $B_\text{parallel}$ values, showing DIRECT's robustness under label noise.

\textbf{Different Levels of Label Noise} As shown in Figures~\ref{fig:noise_level}(a)-(c) and~\ref{fig:intro}(b), we observe the results on imbalanced CIFAR-100 with two classes across numerous levels of label noise, with $0\%$, $10\%$, $15\%$ and $20\%$ respectively. In fact, the noiseless experiment in Figure~\ref{fig:noise_level}(a) is the only setting DIRECT slightly underperforms GALAXY in terms of generalization accuracy. However, we see DIRECT becomes more label-efficient under label noise. It is also worth noting that with high label noise of $20\%$, we observe in Figure~\ref{fig:intro}(b) that existing algorithms underperform random sampling. In contrast, DIRECT significantly outperforms random sampling, saving more than $60\%$ of the annotation cost.

\subsection{Qualitative Analysis}
The ImageNet-LT dataset is constructed with class frequencies that gradually decrease according to class index. In Figure~\ref{fig:perclass}, we organize classes into bins ordered from most frequent to least frequent. The results show that DIRECT significantly outperforms baseline algorithms on less frequent classes (indices 301-1000), which accounts for DIRECT's superior overall performance despite modest accuracy reductions on more frequent classes. This finding aligns with our balancedness analysis, where DIRECT demonstrates improved labeling of samples from rare classes.

\section{Conclusion and Future Work} \label{sec:future}
In this paper, we conducted the first study of deep active learning under both class imbalance and label noise. We proposed an algorithm DIRECT that significantly and consistently outperforms existing literature. In this work, we also addressed the batch sampling problem of \citet{zhang2022galaxy}, by annotating multiple examples in parallel. Studying asynchronous labeling could be a natural extension of our work. A potential solution is to utilize an asynchronous variant of one-dimensional active learning algorithm. In addition, one can further batch the labeling process across different classes to further accommodate an even larger number of parallel annotators.

\clearpage
\section*{Impact Statement} \label{apx:impact}
In the rapidly evolving landscape of machine learning, the efficacy of active learning in addressing data imbalance and label noise is a significant stride towards more robust and equitable AI systems. This research explores how active learning can effectively mitigate the challenges posed by imbalanced datasets and erroneous labels, prevalent in real-world scenarios.

The positive impacts of this research are multifaceted. It enhances the accessibility and utility of machine learning in domains where data imbalance is a common challenge, such as healthcare, finance, and social media analytics. By improving class-balancedness in annotated sets, models trained on these datasets are less biased and more representative of real-world distributions, leading to fairer and more accurate outcomes. Additionally, this research contributes to reducing the time and cost associated with data annotation, which is particularly beneficial in fields where expert annotation is expensive or scarce.

However, if not carefully implemented, active learning strategies could inadvertently introduce new biases or amplify existing ones, particularly in scenarios where the initial data is severely imbalanced or contains deeply ingrained biases. Furthermore, the advanced nature of these techniques may widen the gap between organizations with access to state-of-the-art technology and those without, potentially exacerbating existing inequalities in technology deployment.

\section*{Acknowledgements}
This work has been supported in part by NSF Award 2112471.

\bibliography{sample}
\bibliographystyle{icml2025}

\newpage
\onecolumn
\appendix
\section{Equivalent Objective} \label{apx:proof}
\begin{lemma}
    The agnostic active learning reduction is equivalently finding the optimal separation threshold. Namely,
    \begin{align*}
        \argmin_j \mc{L}(h_j) = \argmax_j \left(\lvert\{y_{(i)} = 1 : 1\leq i\leq j\}\rvert - \lvert\{y_{(i)} = 2 : 1 \leq i\leq j\}\rvert\right)
    \end{align*}
\end{lemma}
\begin{proof}
    Recall the definitions: $h_j(q) = \begin{cases}
    1 & \text{if} \, q \leq q_{(j)} \\
    2 & \text{if} \, q > q_{(j)}
\end{cases}$ and $\mc{L}(h_j) = \sum_{i=1}^N \1\{h_j(q_{(i)}) \neq y_{(i)}\}$, we can expand the loss as follows
\begin{align*}
    \argmin_j \mc{L}(h_j) &= \argmin_j \sum_{i=1}^N \1\{h_j(q_{(i)}) \neq y_{(i)}\} \\
    &= \argmin_j N - \sum_{i=1}^N \1\{h_j(q_{(i)}) = y_{(i)}\}\\
    &= \argmax_j \sum_{i=1}^N \1\{h_j(q_{(i)}) = y_{(i)}\}\\
    &=  \argmax_j \left(\sum_{i=1}^{j} \1\{y_{(i)} = 1\}\right) + \left(\sum_{i=j+1}^{N} \1\{y_{(i)} = 2\}\right) \\
    &=  \argmax_j \left(\sum_{i=1}^{j} \1\{y_{(i)} = 1\}\right) + \left(\sum_{i=j+1}^{N} \1\{y_{(i)} = 2\}\right) -  \left(\sum_{i=1}^{N} \1\{y_{(i)} = 2\}\right)\\
    &=  \argmax_j \sum_{i=1}^{j} \left(\1\{y_{(i)} = 1\} - \1\{y_{(i)} = 2\}\right)
\end{align*}
\end{proof}

\section{Theoretical Analysis}\label{apx:theory}
In this section, we analyze the performance of GALAXY under random label noise and show the probability of identifying and sampling around additional cuts increases as more examples are labeled. This is in contrast to the DIRECT's agnostic active learning approach, where the probability of identifying and sampling around only the optimal separation threshold decays exponentially in the number of labeling budget.

Specifically, under the binary classification scenario, one is given a sorted list of $N$ examples $\{x_{(i)}\}_{i=1}^N$, with ground truth labels $y_{(1)}^\star = y_{(2)}^\star = ... = y_{(N_1)}^\star = 1$ and $y_{(N_1 + 1)}^\star = ...= y_{(N_1 + N_2)}^\star = 2$, where $N_1 + N_2 = N$. Under uniform i.i.d. label noise with noise ratio $\eta > 0$, the \emph{observed labels} are denoted as $\{y_{(i)}\}_{i=1}^N$, where $\P(y_{(i)} \neq y_{(i)}^\star) = \eta$. In other words, the observed label is flipped with probability $\eta$.

\begin{theorem}
    Given a budget of $b > 2\log N$, let $M_b$ be the random variable denoting number of identified cuts in addition to the optimal separation threshold by one round of GALAXY. We must have $\P(M_b \geq 1) \geq 1-\exp(-b\log(\frac{1}{1-\eta})/2)$, implying GALAXY samples around at least one more cut in addition to the optimal separation threshold with high probability.
\end{theorem}
\begin{proof}
    In the perfect scenario where GALAXY does not receive any corrupted labels, it would use $\log N$ budget with bisection to find the optimal separation threshold and annotate around it.
    However, within the first $\frac{b}{2}$ annotations, whenever GALAXY receives a corrupted label, it will identify a cut in addition to the optimal separation threshold, i.e., $M_b \geq 1$. Therefore, the probability of $M_b \geq 1$ is greater than the probability of receiving at least one corrupted labels in the first $\frac{b}{2}$ annotations. With simple probability bound, we can show that
    \begin{align*}
        \P(M_b \geq 1) > 1 - (1-\eta)^{b/2} = 1 - \exp(b\log(1-\eta)/2) = 1 - \exp(-b\log(\frac{1}{1-\eta})/2).
    \end{align*}
\end{proof}
As the theorem suggests, when $b$ is large, GALAXY will identify and annotate around at least one additional cut with high probability.

\section{Time Complexity} \label{apx:complexity}
The computation complexity for each batch of DIRECT is $O(KN\log(N) + B_{\text{train}} N)$ for data selection plus the training and inference costs of the neural network. $O(KN\log(N))$ comes from sorting examples by their margin scores for each class and $O(B_{\text{train}} N)$ is the cost for running Algorithm 2 for $O(B_{\text{train}})$ iterations. Each iteration of Algorithm 2 only costs $O(N)$ time as we can efficiently solve the objective by cumulative sums. We note that the cost associated with neural network training and inference is always the dominating factor. 

For comparisons, BADGE has time complexity $O(B_{\text{train}}N(K + D))$, significantly more expensive than DIRECT, with $D$ denotes the dimensionality of the penultimate layer features. In addition, GALAXY has computational complexity of $O(KN\log(N)) + B_{\text{train}} KN$, also more expensive than DIRECT. In all of our experiments, both BADGE and GALAXY indeed is slower than DIRECT. We further note that the time complexity factor of $K$ in DIRECT can be easily parallelized by conducting the $K$ sorting procedures on different CPU cores.

Below, we also provide a comprehensive list of computational complexity of different algorithms we consider in our implementation:
As for data selection algorithms, let $K$ be the number of classes, $N$ be the pool size and $B_{\text{train}}$ be the batch size, $D$ be the penultimate layer embedding dimension and $T$ be the number of batches. Below, we detail the computation cost of data selection of each algorithm we consider.
\begin{itemize}
    \item DIRECT: $O(T(KN\log N + B_{train}N))$.

    \item GALAXY: $O(T(KN\log N + B_{train}KN))$

    \item BADGE: $O(TB_{train}N(K + D))$

    \item Margin sampling/most likely positive/confidence sampling: $O(TKN)$

    \item Coreset: $O(T^2B_{train}ND)$
    
    \item SIMILAR: $O(TB_{train}ND)$

    \item Cluster margin: $O(N^2\log N + TN(K + \log N))$

    \item BASE: $O(TN(D+B_{train}))$
\end{itemize}

Overall, our experiments are conducted on NVIDIA 3090 ti GPUs. Each trial of the ResNet-18 experiment takes less than two hours while each trial of the LabelBench experiments takes roughly 12 hours.

\section{Experiment Setup} \label{sec:setup}
\begin{table}[t]
    \begin{center}
    \begin{small}
    \begin{tabular}{lccc}
    \toprule
    Name & \shortstack{$K$} & $N$ & \shortstack{Imb Ratio\\$\gamma = \frac{\min_k N_k}{\max_{k'} N_{k'}}$}  \\
    \midrule
    Imb CIFAR-10 & $2$ & $50000$ & $.1111$\\
    Imb CIFAR-10 & $3$ & $50000$ & $.1250$\\
    Imb CIFAR-100 & $2$ & $50000$ & $.0101$\\
    Imb CIFAR-100 & $3$ & $50000$ & $.0102$\\
    Imb CIFAR-100 & $10$ & $50000$ & $.0110$\\
    Imb SVHN & $2$ & $73257$ & $.0724$\\
    Imb SVHN & $3$ & $54448$ & $.2546$\\
    PathMNIST & $2$ & $89996$ & $.1166$\\
    FMoW & $62$ & $76863$ & $.0049$\\
    iWildCam & $14$ & $129809$ & $4.57 \cdot 10^{-5}$\\
    \bottomrule
    \end{tabular}
    \end{small}
    \end{center}
    \caption{Dataset settings for our experiments. $N$ denotes the total number of examples in our dataset. $\gamma$ is the class imbalance ratio defined in Section~\ref{sec:notation}.}
    \label{tab:dataset}
\end{table}
\textbf{ResNet-18 Experiments.} ResNet-18 with passive training has been the standard evaluation in existing deep active literature~\citep{ash2019deep,zhang2022galaxy}. Our experiment setup utilizes the CIFAR-10, CIFAR-100~\citep{krizhevsky2009learning}, SVHN~\citep{netzer2011reading} and PathMNIST~\citep{yang2021medmnist} image classification datasets. The original forms of these datasets are roughly balanced across 9, 10 or 100 classes. We construct an extremely imbalanced dataset by grouping a large number of classes into one majority class. For example, given a balanced dataset above with $M$ classes. We generate an imbalanced dataset with $K$ classes ($K < M$) by the first $K - 1$ classes from the original dataset and combining the rest of the classes $K,...,M$ into a single majority class $K$. Imbalance ratios are shown in Table~\ref{tab:dataset}. In addition, we also utilize the standard CIFAR-10LT and CIFAR-100LT variants in our experiments for noisy label setting.

For neural network training, we utilize the standard passive training on labeled examples with cross entropy loss and Adam optimizer~\citep{kingma2014adam}. The ResNet-18 model~\citep{he2016deep} is pretrained on ImageNet~\citep{deng2009imagenet} from the PyTorch library. To address data imbalance, for all algorithms, we utilize a reweighted cross entropy loss by the inverse frequency of the number of labeled examples in each class. For experiments with label noise, we further add a $10\%$ label smoothing during training~\citep{muller2019does} for all algorithms.

\textbf{LabelBench Experiments.} Proposed by \citet{zhang2024labelbench}, LabelBench evaluates active learning performance in a more comprehensive framework. Here, we fine-tune the large pretrained model from CLIP's ViT-B32 model~\citep{radford2021learning}. The framework also utilizes semi-supervised learning method FlexMatch~\citep{zhang2021flexmatch} to further leverage the unlabeled examples in the pool for training. We conduct experiments on the two imbalanced datasets in LabelBench, with FMoW~\citep{christie2018functional} and iWildcam~\citep{beery2021iwildcam}. Similar to the ResNet-18 experiments, for all algorithms, we use a $10\%$ label smoothing in the loss function to improve training under label noise. We did find FlexMatch to perform poorly under the combination of imbalance and label noise, so we used the passive training method for label noise experiments.

\section{All Results} \label{apx:all_results}
\subsection{Noiseless Results under Imbalance} \label{apx:noiseless}
\begin{figure*}[h!]
    \centering
    \begin{subfigure}[t]{.49\textwidth}
        \centering
        \includegraphics[width=\textwidth]{figures/cifar_unbalanced_2_accuracy.pdf}
    \end{subfigure}
    \begin{subfigure}[t]{.49\textwidth}
        \centering
        \includegraphics[width=\textwidth]{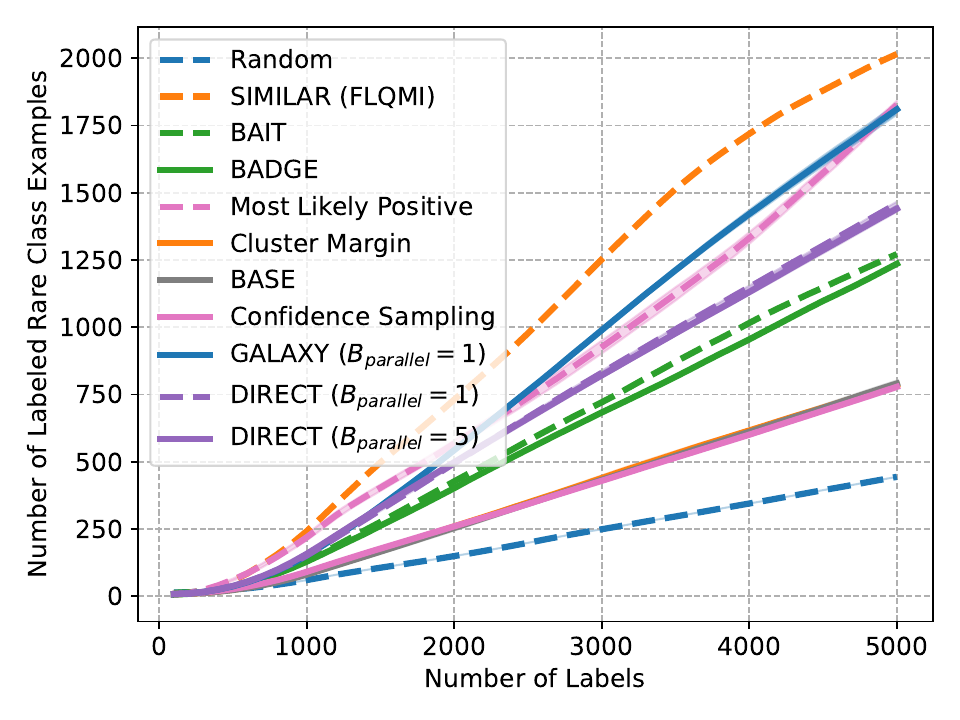}
    \end{subfigure}
    \caption{Imbalanced CIFAR-10, two classes.}
\end{figure*}
\begin{figure*}[h!]
    \centering
    \begin{subfigure}[t]{.49\textwidth}
        \centering
        \includegraphics[width=\textwidth]{figures/cifar_unbalanced_3_accuracy.pdf}
    \end{subfigure}
    \begin{subfigure}[t]{.49\textwidth}
        \centering
        \includegraphics[width=\textwidth]{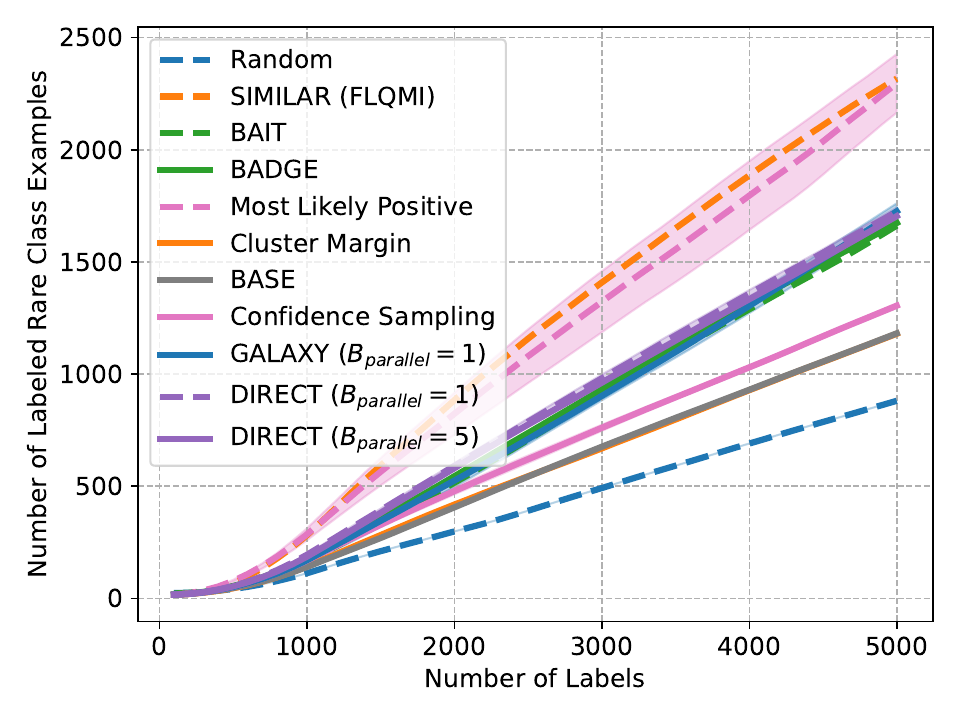}
    \end{subfigure}
    \caption{Imbalanced CIFAR-10, three classes.}
\end{figure*}
\begin{figure*}[h!]
    \centering
    \begin{subfigure}[t]{.49\textwidth}
        \centering
        \includegraphics[width=\textwidth]{figures/cifar100_unbalanced_2_accuracy.pdf}
    \end{subfigure}
    \begin{subfigure}[t]{.49\textwidth}
        \centering
        \includegraphics[width=\textwidth]{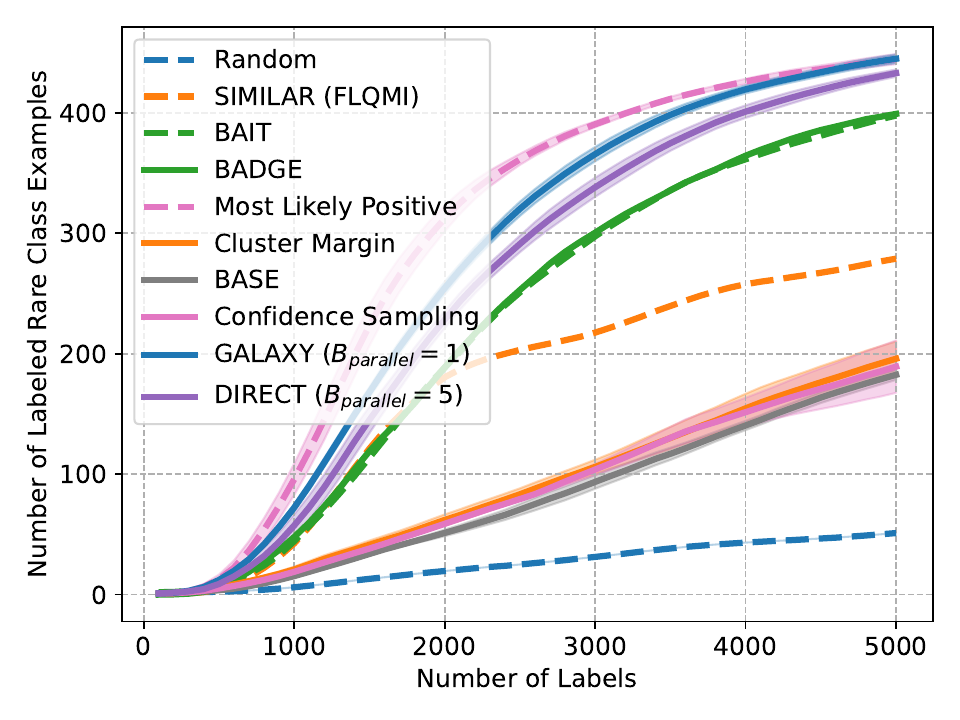}
    \end{subfigure}
    \caption{Imbalanced CIFAR-100, two classes.}
\end{figure*}
\begin{figure*}[h!]
    \centering
    \begin{subfigure}[t]{.49\textwidth}
        \centering
        \includegraphics[width=\textwidth]{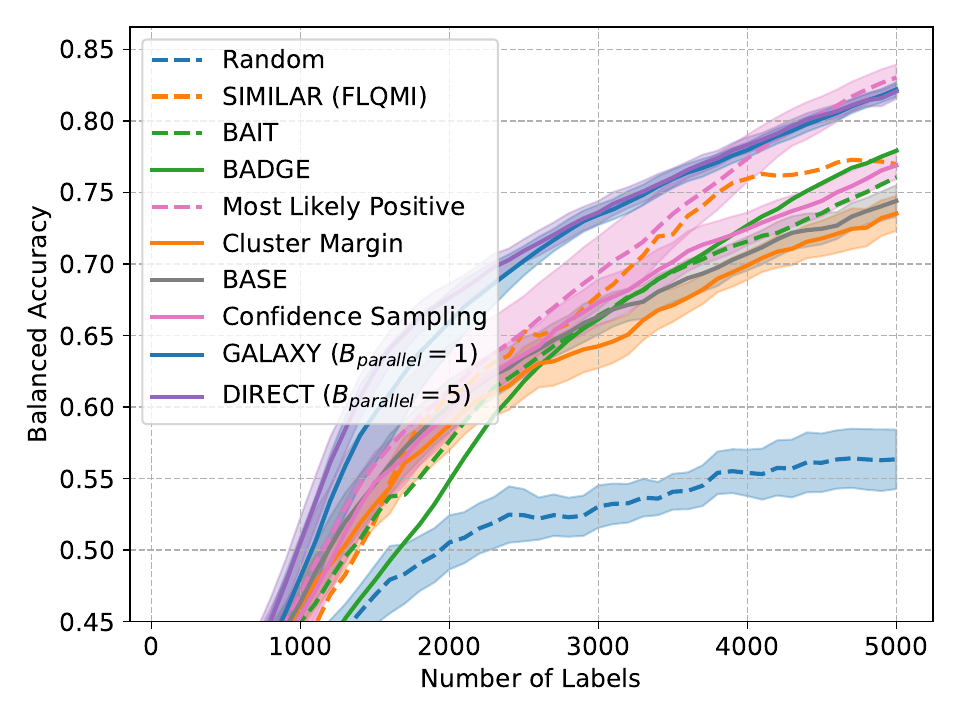}
    \end{subfigure}
    \begin{subfigure}[t]{.49\textwidth}
        \centering
        \includegraphics[width=\textwidth]{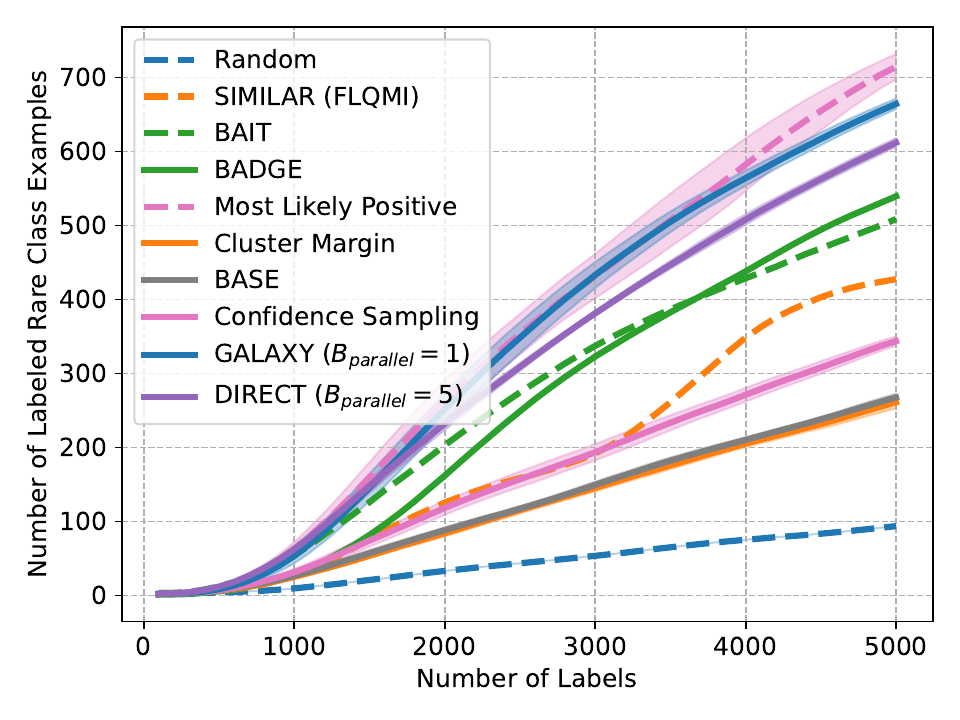}
    \end{subfigure}
    \caption{Imbalanced CIFAR-100, three classes.}
\end{figure*}
\begin{figure*}[h!]
    \centering
    \begin{subfigure}[t]{.49\textwidth}
        \centering
        \includegraphics[width=\textwidth]{figures/cifar100_unbalanced_10_accuracy.pdf}
    \end{subfigure}
    \begin{subfigure}[t]{.49\textwidth}
        \centering
        \includegraphics[width=\textwidth]{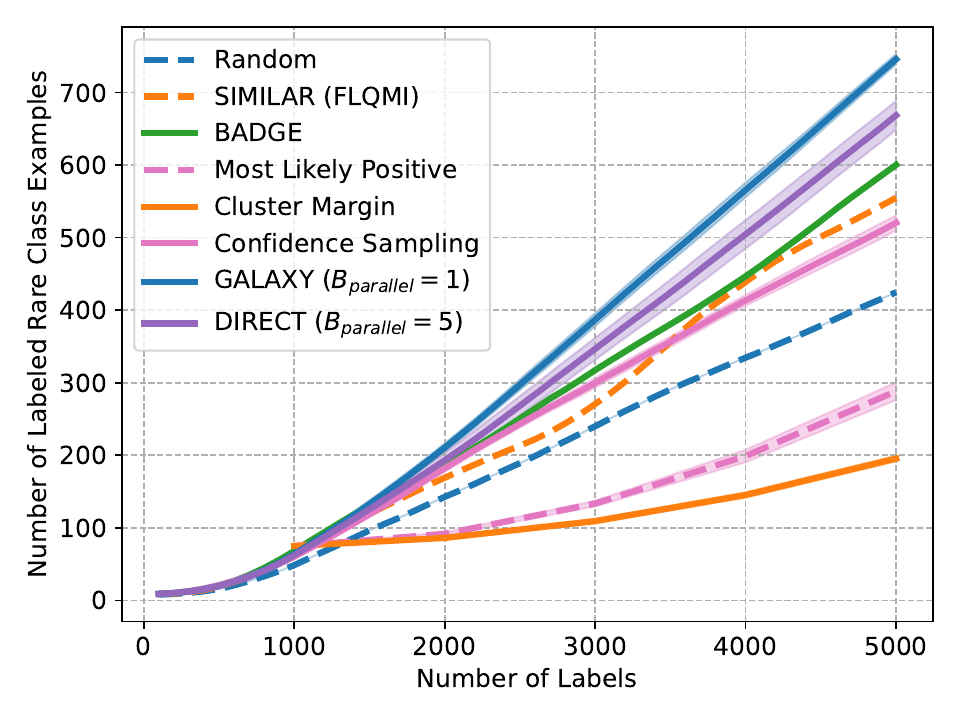}
    \end{subfigure}
    \caption{Imbalanced CIFAR-100, 10 classes.}
\end{figure*}
\begin{figure*}[h!]
    \centering
    \begin{subfigure}[t]{.49\textwidth}
        \centering
        \includegraphics[width=\textwidth]{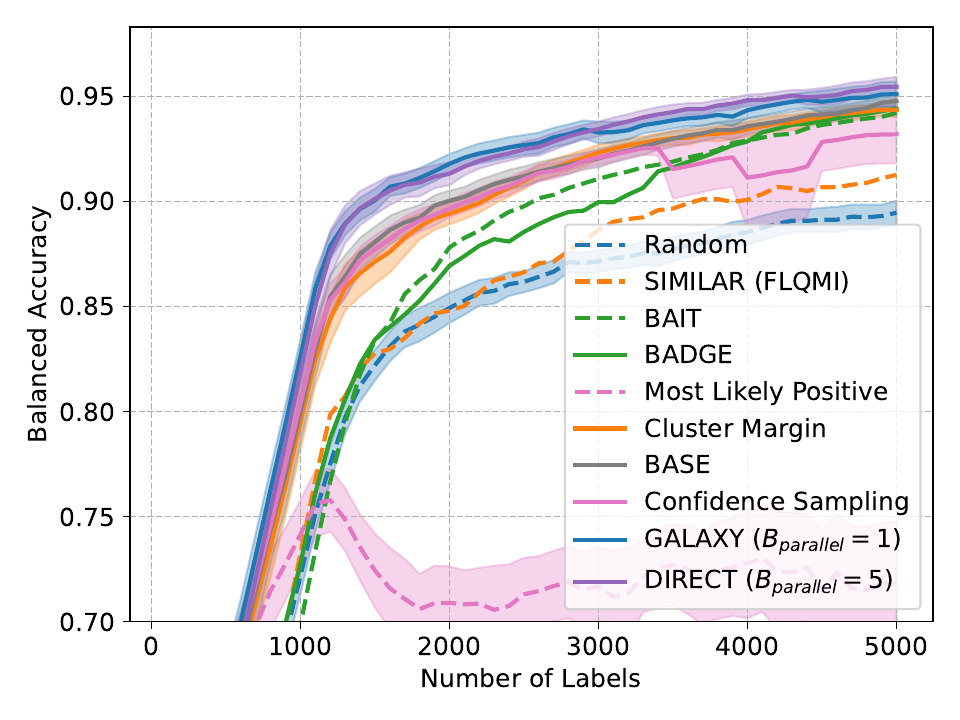}
    \end{subfigure}
    \begin{subfigure}[t]{.49\textwidth}
        \centering
        \includegraphics[width=\textwidth]{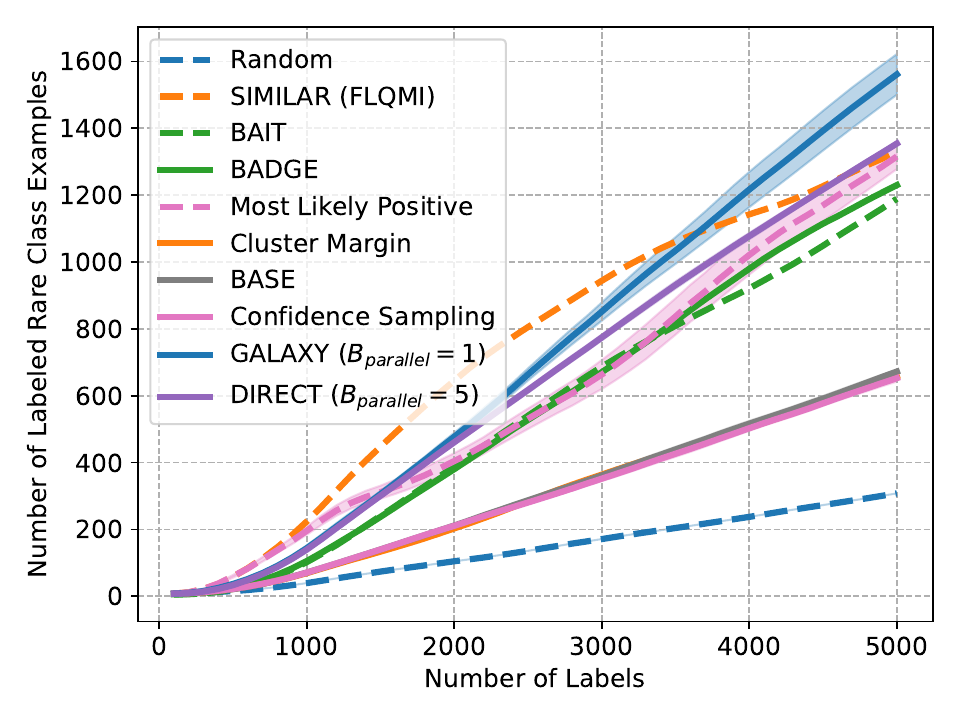}
    \end{subfigure}
    \caption{Imbalanced SVHN, two classes.}
\end{figure*}
\begin{figure*}[h!]
    \centering
    \begin{subfigure}[t]{.49\textwidth}
        \centering
        \includegraphics[width=\textwidth]{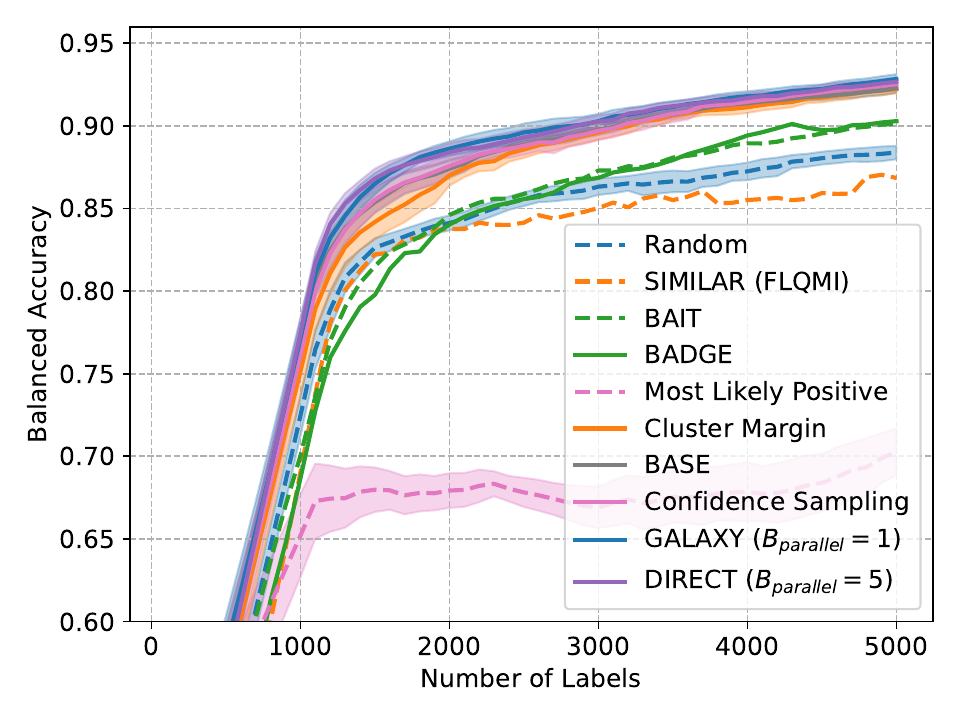}
    \end{subfigure}
    \begin{subfigure}[t]{.49\textwidth}
        \centering
        \includegraphics[width=\textwidth]{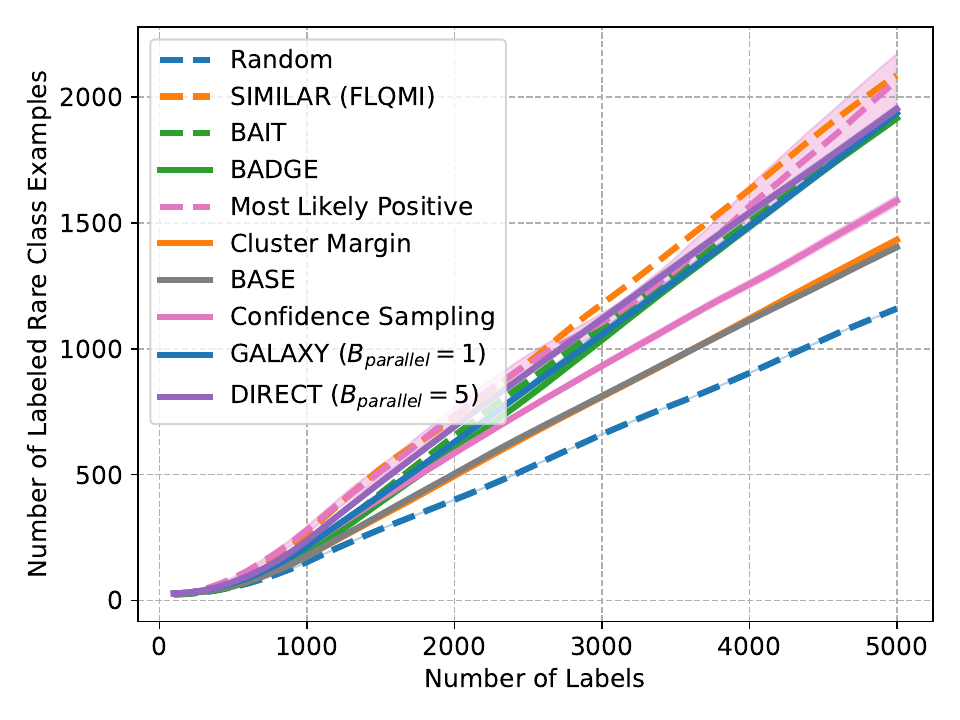}
    \end{subfigure}
    \caption{Imbalanced SVHN, three classes.}
\end{figure*}
\begin{figure*}[h!]
    \centering
    \begin{subfigure}[t]{.49\textwidth}
        \centering
        \includegraphics[width=\textwidth]{figures/medmnist_accuracy_legend.pdf}
    \end{subfigure}
    \begin{subfigure}[t]{.49\textwidth}
        \centering
        \includegraphics[width=\textwidth]{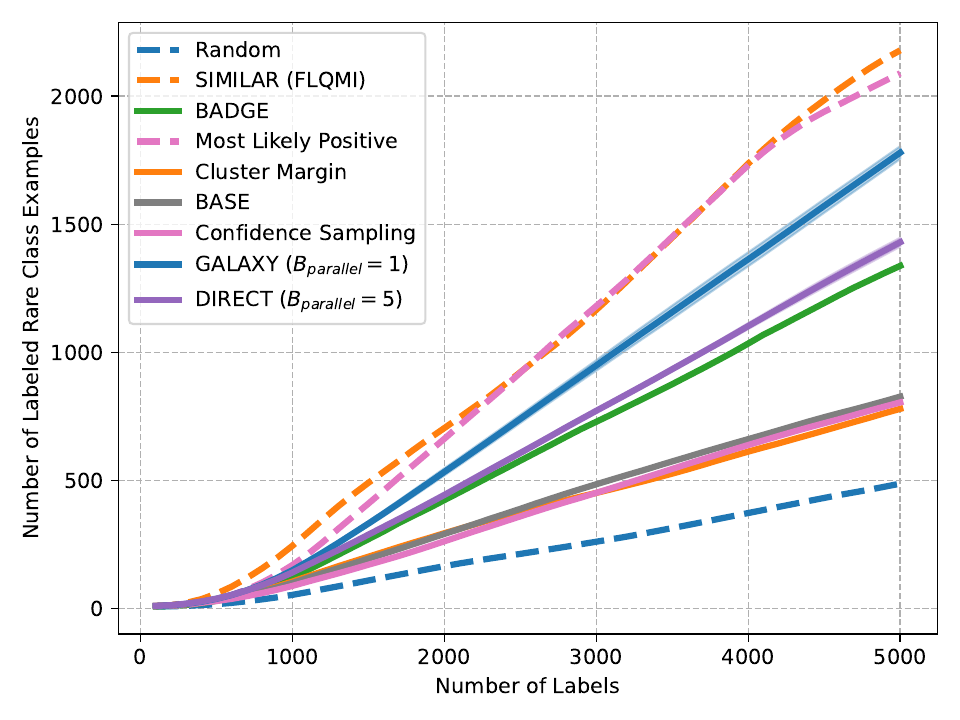}
    \end{subfigure}
    \caption{PathMNIST, two classes.}
\end{figure*}
\begin{figure*}[h!]
    \centering
    \begin{subfigure}[t]{.49\textwidth}
        \centering
        \includegraphics[width=\textwidth]{figures/fmow_3000_none_clip_ViTB32_flexmatch_Balanced_Training_Accuracy.pdf}
        \caption{FMoW Balanced Pool Accuracy}
    \end{subfigure}
    \begin{subfigure}[t]{.49\textwidth}
        \centering
        \includegraphics[width=\textwidth]{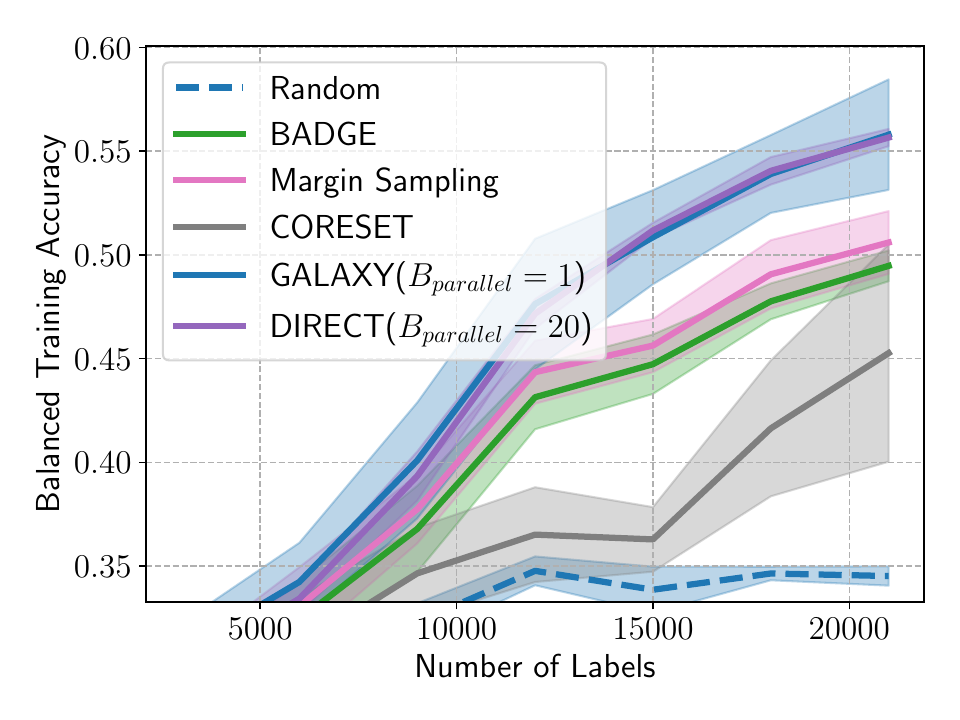}
        \caption{iWildcam Balanced Pool Accuracy}
    \end{subfigure}
    \begin{subfigure}[t]{.49\textwidth}
        \centering
        \includegraphics[width=\textwidth]{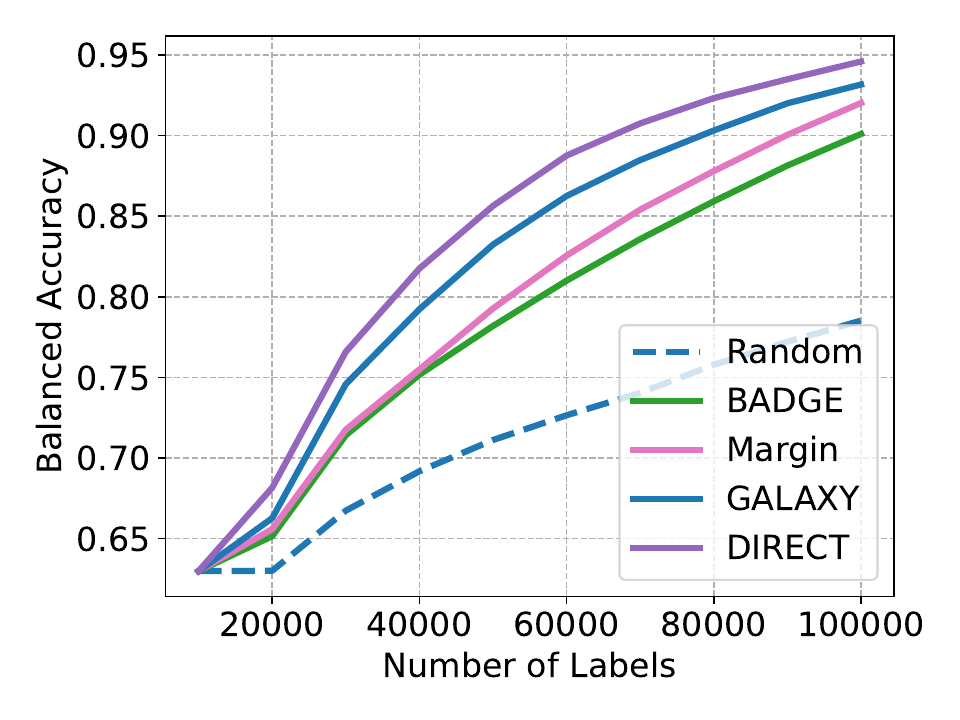}
        \caption{ImageNet-LT Balanced Pool Accuracy}
    \end{subfigure}
    \begin{subfigure}[t]{.49\textwidth}
        \centering
        \includegraphics[width=\textwidth]{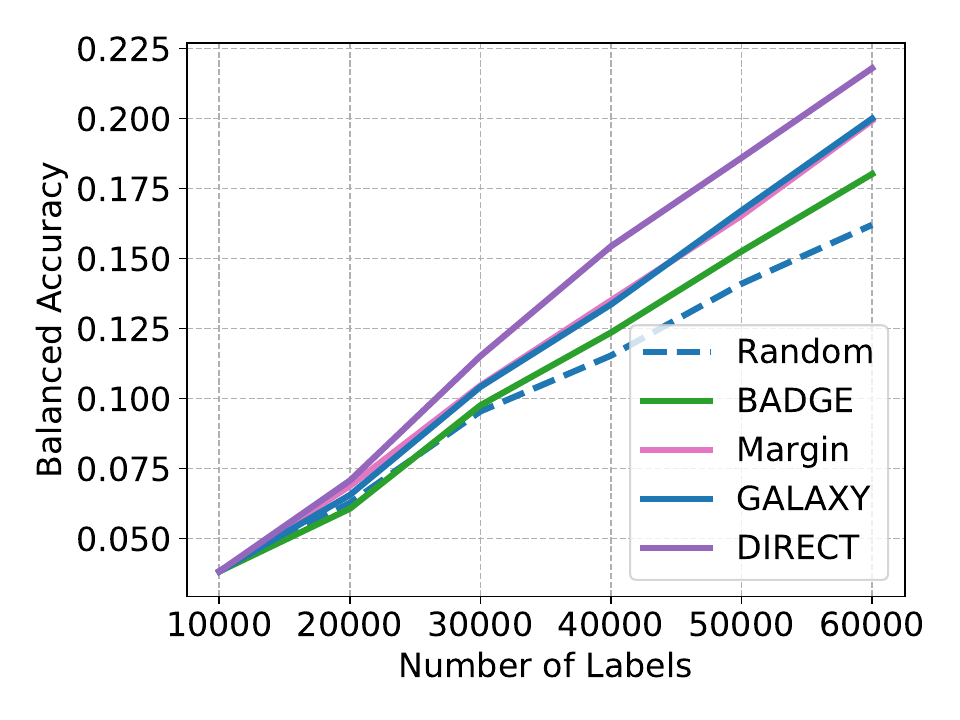}
        \caption{iNaturalist Balanced Pool Accuracy}
    \end{subfigure}

    \caption{LabelBench results in the noiseless setting.}
\end{figure*}

\newpage
\subsection{Label Noise Results under Imbalance} \label{apx:noisy}
\begin{figure*}[h!]
    \centering
    \begin{subfigure}[t]{.49\textwidth}
        \centering
        \includegraphics[width=\textwidth]{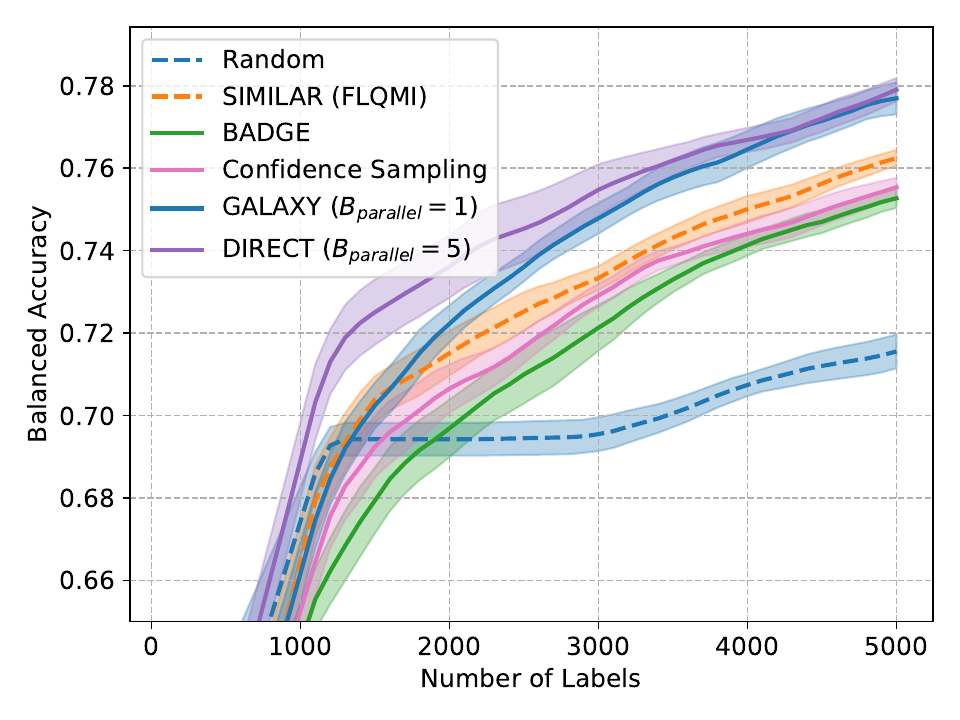}
    \end{subfigure}
    \begin{subfigure}[t]{.49\textwidth}
        \centering
        \includegraphics[width=\textwidth]{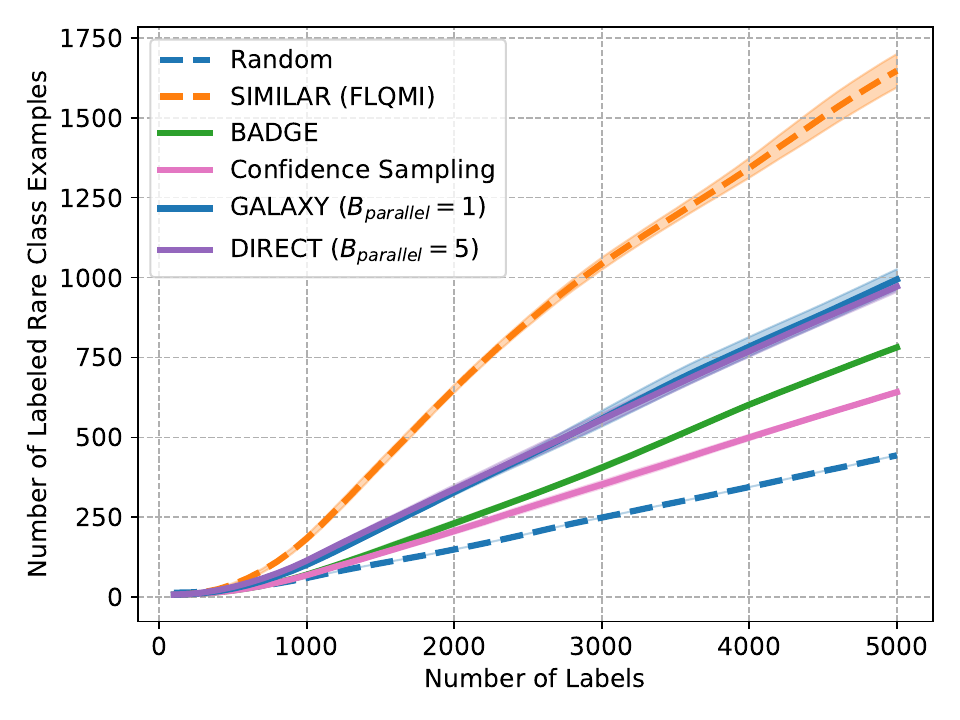}
    \end{subfigure}
    \caption{Imbalanced CIFAR-10, two classes, 10\% label noise.}
\end{figure*}
\begin{figure*}[h!]
    \centering
    \begin{subfigure}[t]{.49\textwidth}
        \centering
        \includegraphics[width=\textwidth]{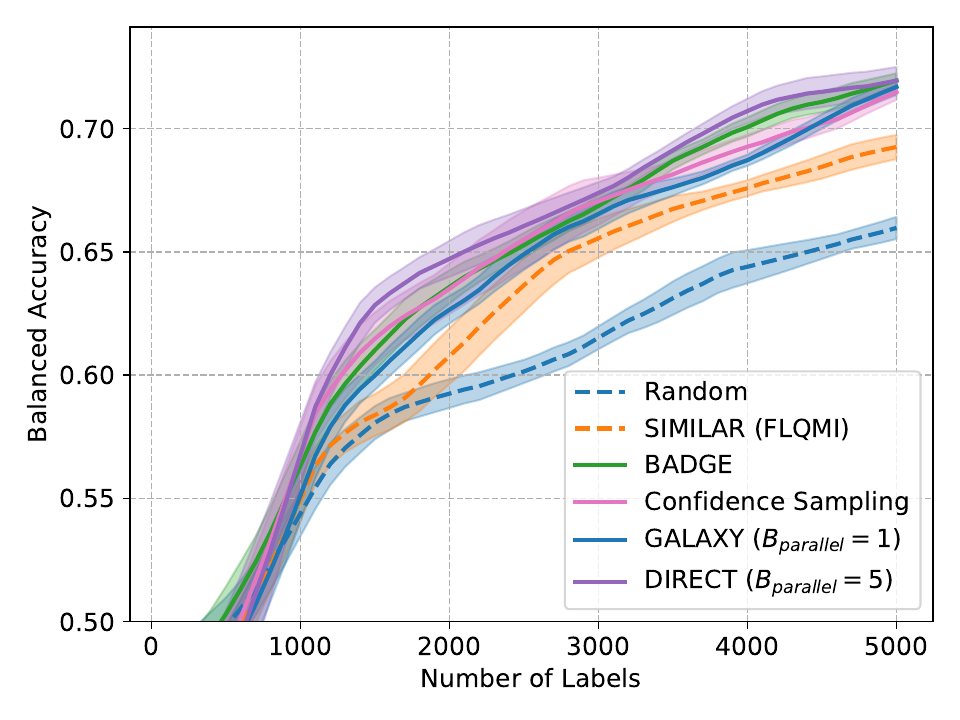}
    \end{subfigure}
    \begin{subfigure}[t]{.49\textwidth}
        \centering
        \includegraphics[width=\textwidth]{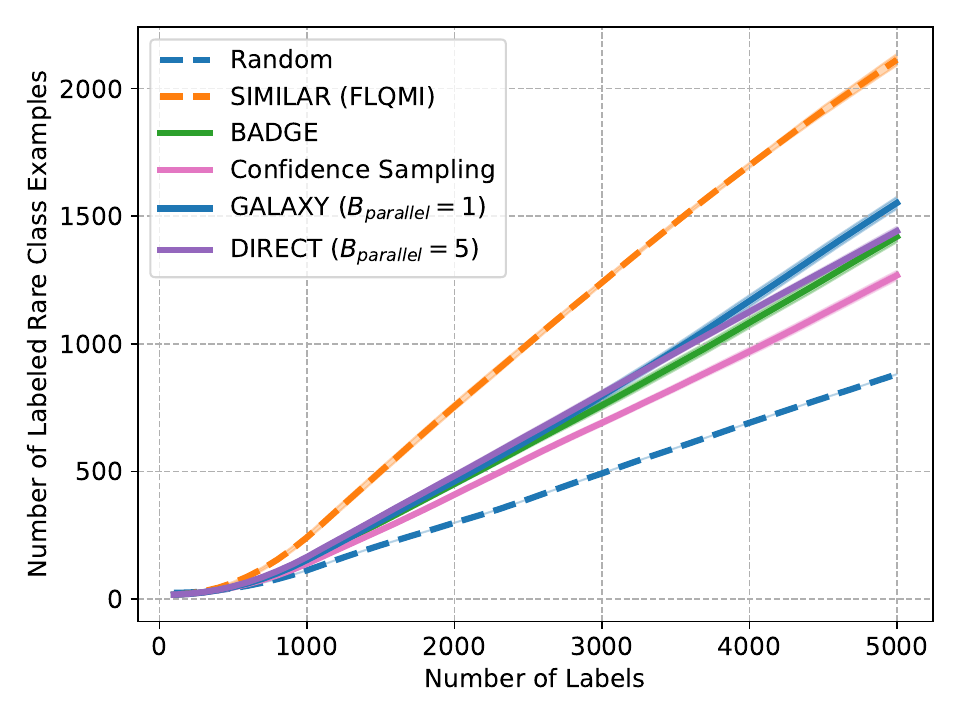}
    \end{subfigure}
    \caption{Imbalanced CIFAR-10, three classes, 10\% label noise.}
\end{figure*}
\begin{figure*}[h!]
    \centering
    \begin{subfigure}[t]{.49\textwidth}
        \centering
        \includegraphics[width=\textwidth]{figures/cifar100_unbalanced_2_label_noise_0.1_accuracy.pdf}
    \end{subfigure}
    \begin{subfigure}[t]{.49\textwidth}
        \centering
        \includegraphics[width=\textwidth]{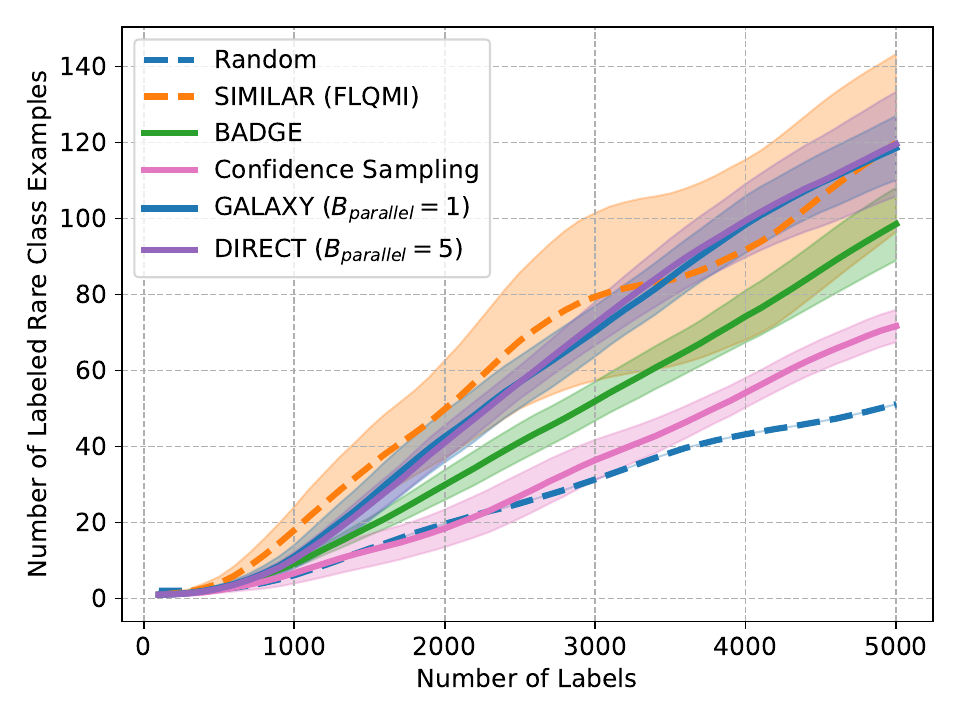}
    \end{subfigure}
    \caption{Imbalanced CIFAR-100, two classes, 10\% label noise.}
\end{figure*}
\begin{figure*}[h!]
    \centering
    \begin{subfigure}[t]{.49\textwidth}
        \centering
        \includegraphics[width=\textwidth]{figures/cifar100_unbalanced_2_label_noise_0.15_accuracy.pdf}
    \end{subfigure}
    \begin{subfigure}[t]{.49\textwidth}
        \centering
        \includegraphics[width=\textwidth]{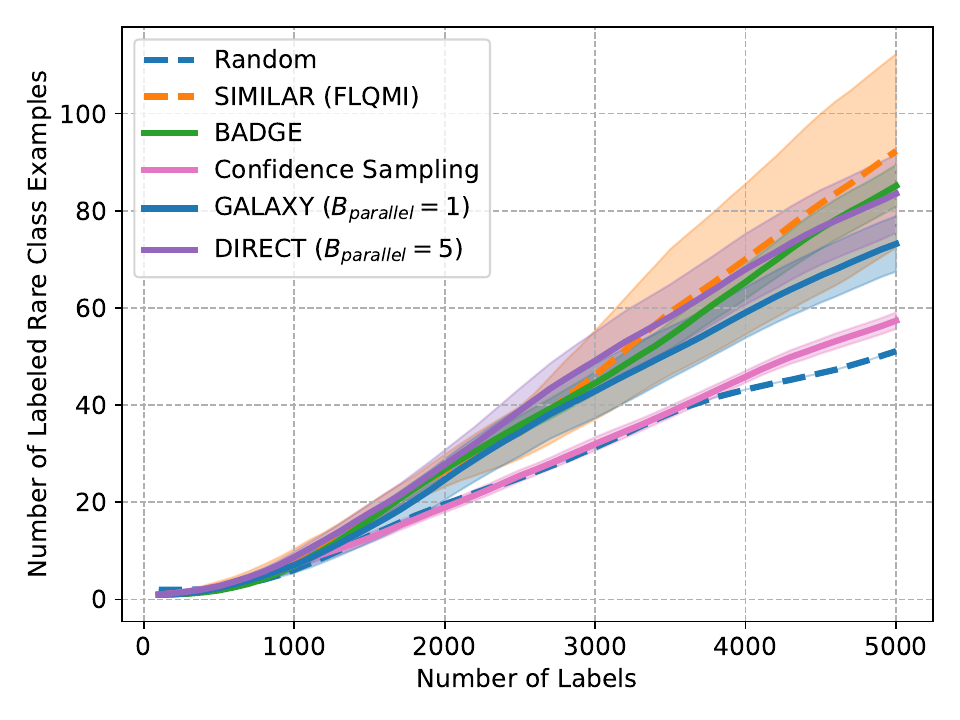}
    \end{subfigure}
    \caption{Imbalanced CIFAR-100, two classes, 15\% label noise.}
\end{figure*}
\begin{figure*}[h!]
    \centering
    \begin{subfigure}[t]{.49\textwidth}
        \centering
        \includegraphics[width=\textwidth]{figures/cifar100_unbalanced_2_label_noise_0.2_accuracy.pdf}
    \end{subfigure}
    \begin{subfigure}[t]{.49\textwidth}
        \centering
        \includegraphics[width=\textwidth]{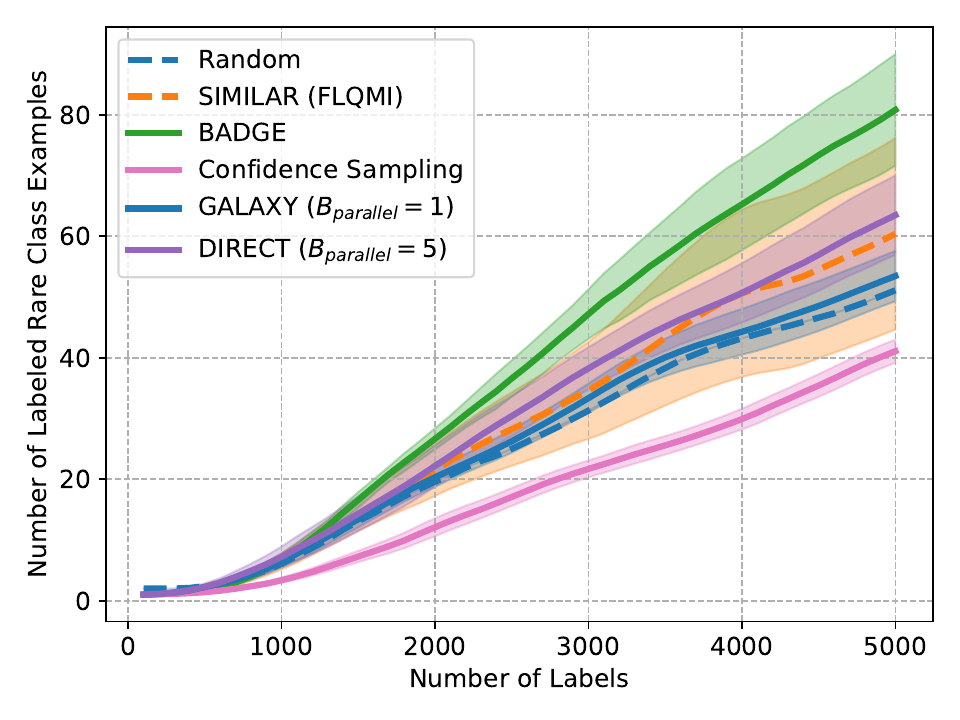}
    \end{subfigure}
    \caption{Imbalanced CIFAR-100, two classes, 20\% label noise.}
\end{figure*}
\begin{figure*}[h!]
    \centering
    \begin{subfigure}[t]{.49\textwidth}
        \centering
        \includegraphics[width=\textwidth]{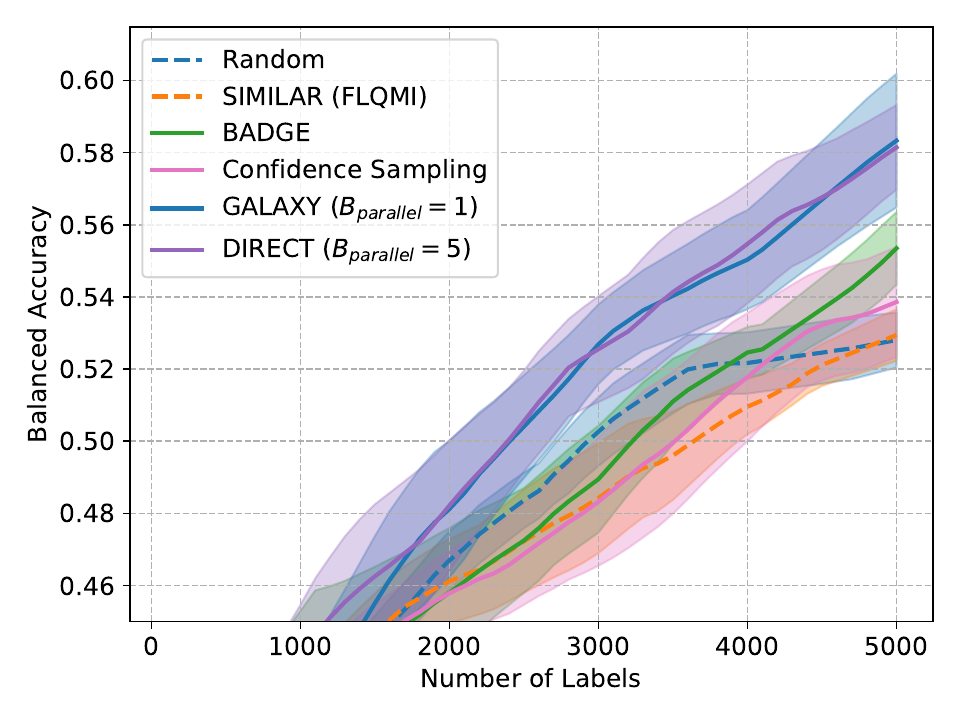}
    \end{subfigure}
    \begin{subfigure}[t]{.49\textwidth}
        \centering
        \includegraphics[width=\textwidth]{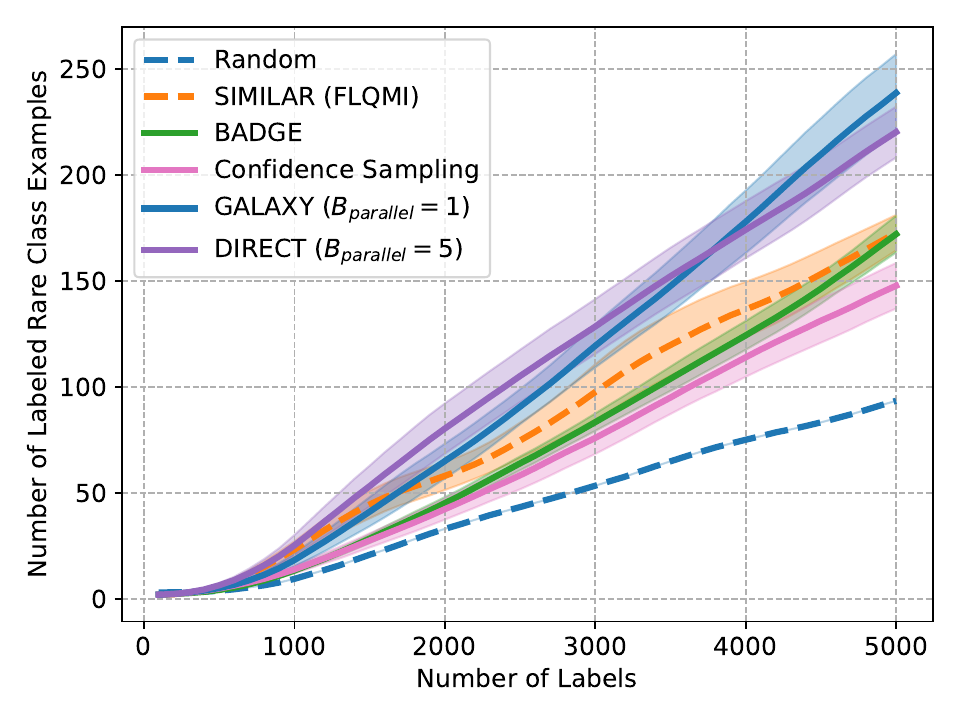}
    \end{subfigure}
    \caption{Imbalanced CIFAR-100, three classes, 10\% label noise.}
\end{figure*}
\begin{figure*}[h!]
    \centering
    \begin{subfigure}[t]{.49\textwidth}
        \centering
        \includegraphics[width=\textwidth]{figures/svhn_unbalanced_2_label_noise_accuracy.pdf}
    \end{subfigure}
    \begin{subfigure}[t]{.49\textwidth}
        \centering
        \includegraphics[width=\textwidth]{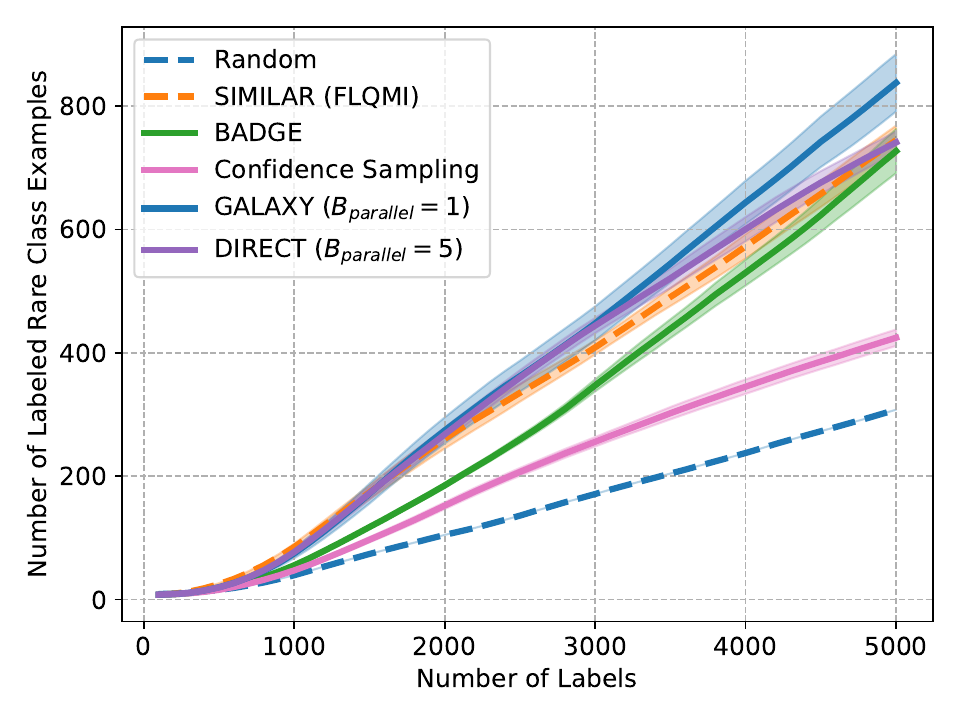}
    \end{subfigure}
    \caption{Imbalanced SVHN, two classes, 10\% label noise.}
\end{figure*}
\begin{figure*}[h!]
    \centering
    \begin{subfigure}[t]{.49\textwidth}
        \centering
        \includegraphics[width=\textwidth]{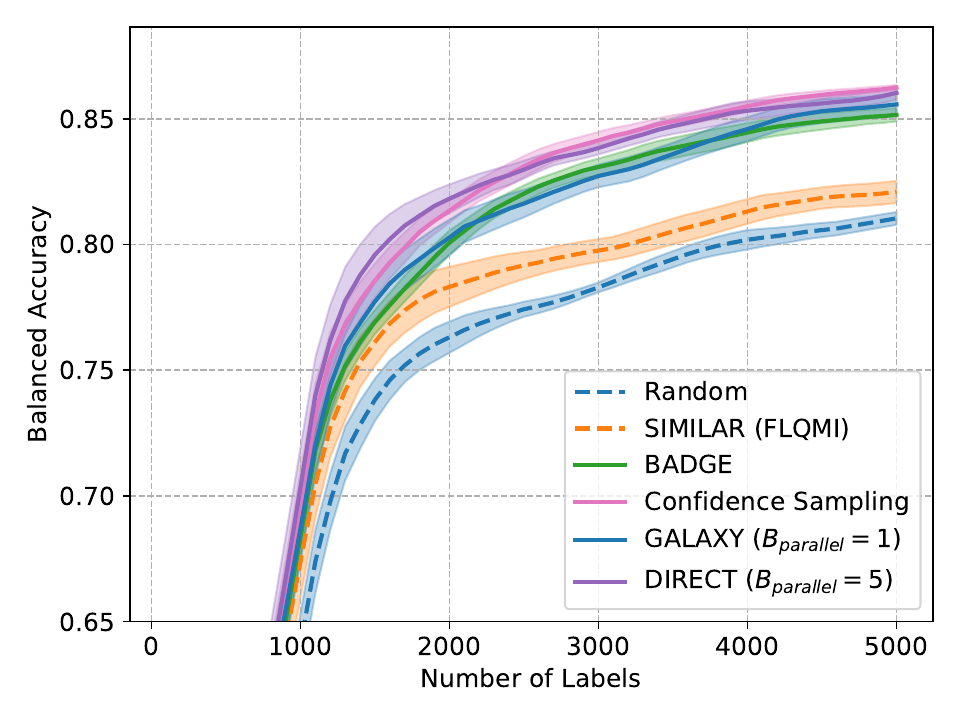}
    \end{subfigure}
    \begin{subfigure}[t]{.49\textwidth}
        \centering
        \includegraphics[width=\textwidth]{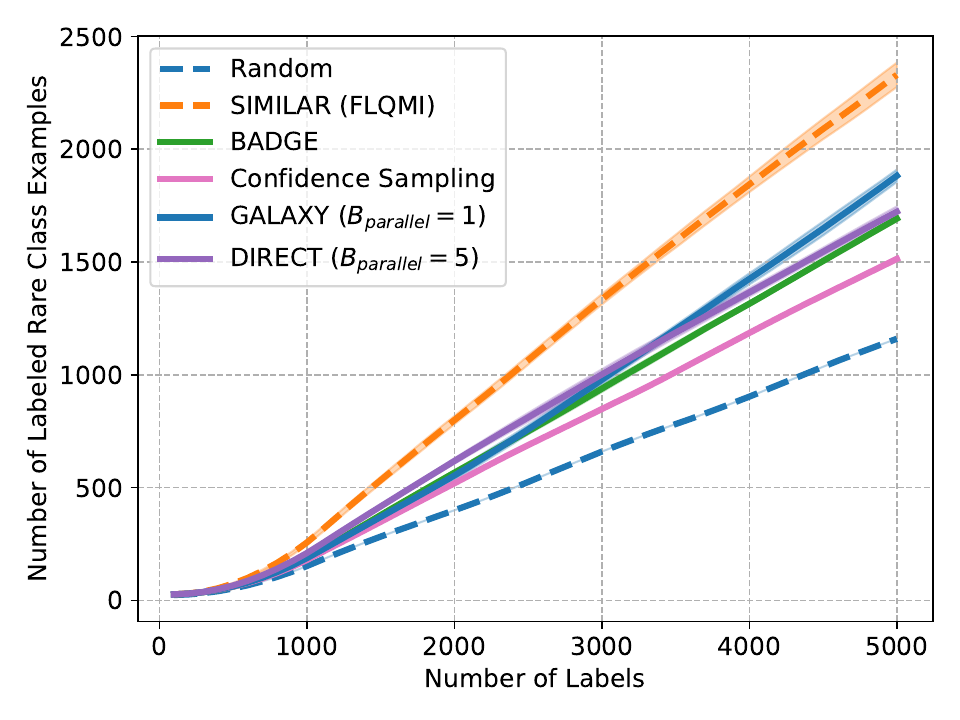}
    \end{subfigure}
    \caption{Imbalanced SVHN, three classes, 10\% label noise.}
\end{figure*}
\begin{figure*}[h!]
    \centering
    \begin{subfigure}[t]{.49\textwidth}
        \centering
        \includegraphics[width=\textwidth]{figures/cifar10lt_label_noise_accuracy.pdf}
        \caption{CIFAR-10LT, 10 classes, 15\% label noise.}
    \end{subfigure}
    \begin{subfigure}[t]{.49\textwidth}
        \centering
        \includegraphics[width=\textwidth]{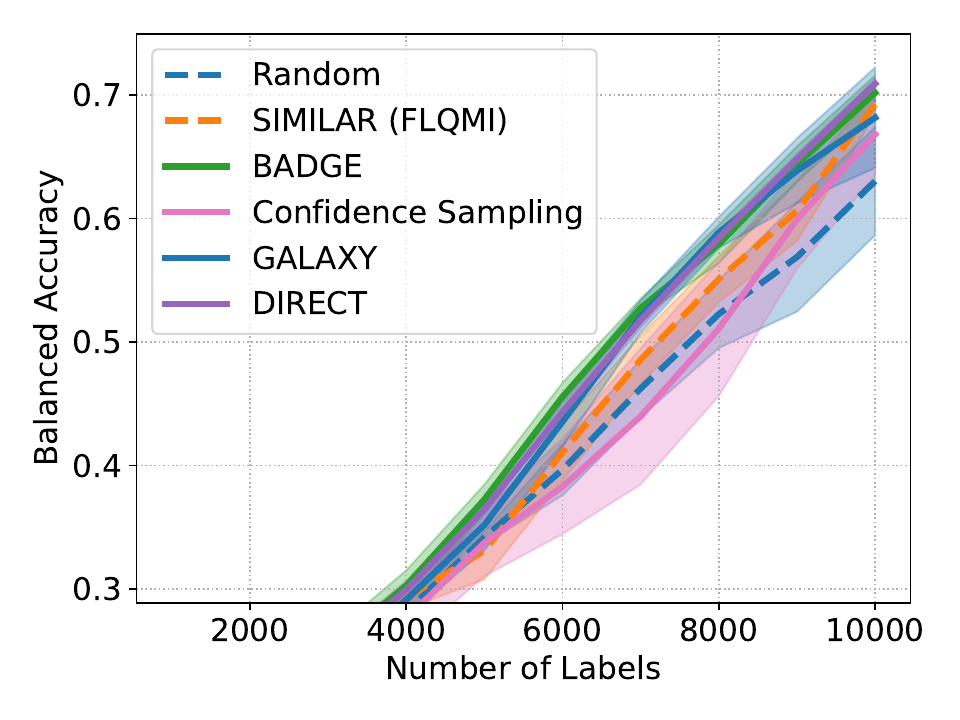}
        \caption{CIFAR-100LT, 100 classes, 15\% label noise.}
    \end{subfigure}
\end{figure*}
\begin{figure*}[h!]
    \centering
    \begin{subfigure}[t]{.49\textwidth}
        \centering
        \includegraphics[width=\textwidth]{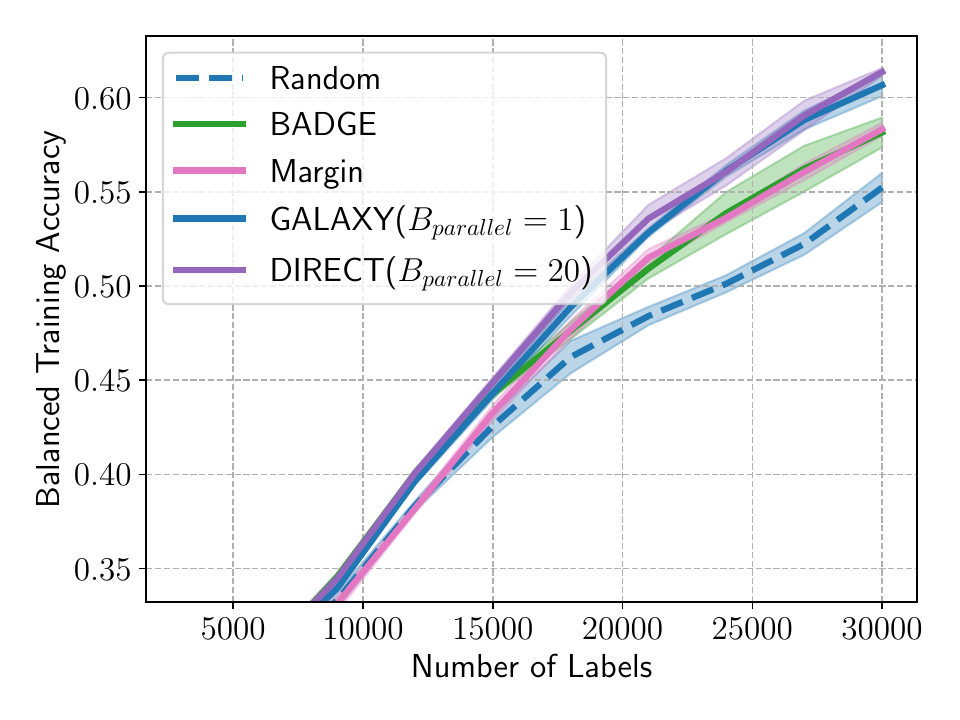}
        \caption{FMoW Balanced Pool Accuracy}
    \end{subfigure}
    \begin{subfigure}[t]{.49\textwidth}
        \centering
        \includegraphics[width=\textwidth]{figures/iwildcam_random_3000_none_clip_ViTB32_passive_Balanced_Training_Accuracy.pdf}
        \caption{iWildcam Balanced Pool Accuracy}
    \end{subfigure}

    \caption{LabelBench results in the 10\% label noise setting.}
\end{figure*}

\end{document}